\documentclass{article}

\usepackage[nonatbib,preprint]{neurips_2023}

\usepackage[utf8]{inputenc} 
\usepackage[T1]{fontenc}    
\usepackage{url}            
\usepackage{booktabs}       
\usepackage{amsfonts,amsthm,amsmath,amssymb}       
\usepackage{nicefrac}       
\usepackage{microtype}      
\usepackage{xcolor}         
\usepackage{optidef}
\usepackage{subcaption}
\usepackage{cite}
\usepackage{algorithm}
\usepackage{algorithmic}
\newtheorem{thm}{Theorem}
\newtheorem{theorem}{Theorem}
\newtheorem{ass}{Assumption}

\newtheorem{lem}{Lemma}

\newcommand{\norm}[1]{\left\lVert#1\right\rVert}

\DeclareMathOperator{\sign}{sign}

\title{A Huber Loss Minimization Approach to Byzantine Robust Federated Learning}

\author{%
  Puning Zhao, Fei Yu, Zhiguo Wan\\
  Zhejiang Lab\\
  Hangzhou, Zhejiang, China\\
  \texttt{\{pnzhao,yufei,wanzhiguo\}@zhejianglab.com} \\
}

\begin{document}

\maketitle

\begin{abstract}
Federated learning systems are susceptible to adversarial attacks. To combat this, we introduce a novel aggregator based on Huber loss minimization, and provide a comprehensive theoretical analysis. Under independent and identically distributed (i.i.d) assumption, our approach has several advantages compared to existing methods. Firstly, it has optimal dependence on $\epsilon$, which stands for the ratio of attacked clients. Secondly, our approach does not need precise knowledge of $\epsilon$. Thirdly, it allows different clients to have unequal data sizes. We then broaden our analysis to include non-i.i.d data, such that clients have slightly different distributions.
\end{abstract}

\section{Introduction}
Due to privacy concerns, there are a large number of isolated information islands, resulting in the difficulty of integrating data from various sources.  Under such background, a novel machine learning framework called Federated Learning (FL) has arisen in recent years \cite{mcmahan2017communication}. FL consists of numerous clients that store and compute data locally, and a central server that plays the role as a coordinator. In comparison to traditional centralized learning, FL offers distinct advantages in terms of both computational efficiency and privacy protection. As a result, FL is gaining increasing attention and has been widely applied in various domains, including mobile devices, industrial engineering, and healthcare \cite{yang2019federated,li2020review}.

Nevertheless, FL is facing several severe challenges \cite{kairouz2021advances}, with one of them being the robustness issue. Due to various factors, including data poisoning, system malfunctions and transmission errors, some clients may send wrong gradient vectors to the server. Consider that these abnormal behaviors are hard to predict and may happen in an unknown manner, it suffices to analyze the most harmful attack, which is typically modeled as Byzantine failure \cite{lamport1982byzantine}. Under this model, an adversary can modify the gradient values uploaded to the master in arbitrary way. Without proper handling, even a single malicious client can significantly degrade the model performance \cite{bagdasaryan2020backdoor}. Therefore, for the safe deployment of FL, it is important to design effective defense strategies robust to Byzantine attacks.

There have been many existing works on Byzantine robust federated learning problems. In particular, various gradient aggregators have been proposed. Krum \cite{blanchard2017machine} picks the gradient vector uploaded from clients with small nearest neighbor distances. However, the global convergence is not guaranteed. \cite{chen2017distributed} proposed a geometric median-of-mean method, which ensures that the model weights converge to a point near to the global minimum, as long as $\epsilon<1/2$, with $\epsilon$ being the fraction of Byzantine machines. This method is not perfect since the error has a suboptimal rate of $\tilde{O}(\sqrt{\epsilon d})$. \cite{yin2018byzantine} analyzed two aggregators. The first one, called coordinate-wise median, is suboptimal if the sample size per client $n$ is smaller than the number of clients $m$. Unfortunately, this is quite likely in practice. The second one is coordinate-wise trimmed mean, which has optimal dependence with $\epsilon$ when $\epsilon$ is small. However, as will be discussed later, if $\epsilon$ is close to $1/2$, then coordinate-wise trimmed mean is not efficient. Another drawback is that this method needs the precise knowledge of $\epsilon$, which is usually not practical. Moreover, the analysis of these previous methods are based on some simplified assumption, including independent and identically distributed (i.i.d) assumption, and that all clients have nearly equal sample sizes. More theoretical analysis is needed in realistic scenarios with heterogeneous and unbalanced data.


In this paper, we propose a novel approach to Byzantine robust federated learning, which aggregates gradients by minimizing a multi-dimensional Huber loss. As a widely used loss function in robust statistics \cite{huber1964robust,huber2004robust,hall1990adaptive,zhao2023robust}, Huber loss combines the advantages of $\ell_1$ and $\ell_2$ loss, and achieves a tradeoff between robustness and consistency. However, the original definition of Huber loss was for scalars. We generalize the original definition, to make it suitable for vectors. In each iteration, given a list of gradient vectors uploaded from clients, the new algorithm obtains the estimated gradient of the underlying global risk function by minimizing the generalized Huber loss, and then use the outcome as the direction of parameter update. 

We then provide theoretical analysis of the proposed method. To begin with, it is assumed that training samples are i.i.d, which is common in most of existing works on Byzantine robust FL \cite{chen2017distributed,blanchard2017machine,yin2018byzantine}. Under i.i.d assumption, two error bounds are derived for balanced and unbalanced data, respectively. Consider that i.i.d assumption may not be realistic, we then make some generalization to allow heterogeneous clients. The result shows that our method is robust to moderate violation of i.i.d assumption. There are several recent works \cite{Li:AAAI:19,pillutla2022robust,ghosh2019robust,wu2020federated,zuo2024byzantine} that deal with heterogeneous data. These works are mainly designed for the case that the differences between clients are large. In slightly heterogeneous regime, these methods can not achieve comparable statistical rates.

Finally, we discuss how to implement our new algorithm. Our design is motivated by Weiszfeld's algorithm for calculating the geometric median of a set of vectors \cite{weiszfeld:AOR:09}. We make some adjustments to Weiszfeld's algorithm, such that it can minimize the multi-dimensional Huber loss. 

In general, compared with existing methods, our new approach has several advantages. Firstly, the dependence of statistical risk on the attack ratio $\epsilon$ is nearly minimax optimal, up to a logarithm factor. Secondly, our method has desirable performance with unbalanced data. In particular, with an adaptive rule of parameter selection, the statistical rate is the same as the case with balanced data. Thirdly, many existing methods require the precise knowledge of $\epsilon$ to set parameters, which is usually not practical. On the contrary, our method works well under the $\epsilon$-agnostic settings.

Our contributions are summarized as follows:
\begin{itemize}
	\item A multi-dimensional Huber loss minimization approach to robust federated learning;
	\item Theoretical analysis of our method for both balanced and unbalanced data under i.i.d assumption, which also provides a guideline of parameter selection;
	\item Extension of the analysis above to heterogeneous clients;
	\item An implementation algorithm of multi-dimensional Huber loss minimization;
	\item Numerical experiments on both synthesized and real data, which validates the effectiveness of our new method.
\end{itemize} 

The remainder of this paper is organized as follows. The problem statement and the proposed method are shown in Section~\ref{sec:method}. The theoretical analysis with and without i.i.d assumption are shown in Section \ref{sec:iid} and \ref{sec:noniid}, respectively. In Section~\ref{sec:implement}, we describe our implementation algorithm. Furthermore, we compare our approach with existing methods in Section~\ref{sec:compare}. Finally, numerical experiments and concluding remarks are shown in Section~\ref{sec:numerical} and \ref{sec:conc}, respectively.

\section{The Proposed Method}\label{sec:method}

In this section, we make a precise statement of the framework of the federated learning problem, and then introduce our proposed aggregator based on minimization of multi-dimensional Huber loss.

The framework is shown in Algorithm \ref{alg:learning}. Suppose there is a server $S_0$ and $m$ clients $S_1,\ldots, S_m$. Denote $\mathcal{B}$ as the set of Byzantine clients. There are $N$ training samples in total, with each client $S_i$ storing $n_i$ of them. Denote $\mathbf{Z}_{ij}$ as the $j$-th sample in the $i$-th client, $j\in \{1,\ldots, n_i \}$. Let $f(\mathbf{w}, \mathbf{z})$ be the loss function of model parameter $\mathbf{w}\in \mathcal{W}$ with respect to sample $\mathbf{z}$, in which $\mathcal{W} \subset \mathbb{R}^d$ is the parameter space. Moreover, define the global risk function as
\begin{eqnarray}
	F(\mathbf{w}) = \mathbb{E}[f(\mathbf{w}, \mathbf{Z})],
\end{eqnarray}
in which $\mathbf{Z}$ follows the global distribution of training samples. In particular, under i.i.d assumption, $\mathbf{Z}$ just follows the distribution of arbitrary $\mathbf{Z}_{ij}$. Otherwise, with heterogeneous clients, $\mathbf{Z}$ follows the average distribution of clients weighted by $n_i$. The goal is to learn the global minimizer
\begin{eqnarray}
	\mathbf{w}^* = \underset{\mathbf{w}\in \mathcal{W}}{\arg\min} F(\mathbf{w}).
	\label{eq:wopt}
\end{eqnarray}

The algorithm starts with an initial parameter $\mathbf{w}_0\in \mathcal{W}$, At each iteration $t=0,1,\ldots$, we find an update $\mathbf{w}_{t+1}$. Ideally, it would be better if we know $\nabla F(\mathbf{w}_t)$, so that we can use a simple gradient descent to update $\mathbf{w}$, i.e. $\mathbf{w}_{t+1} = \mathbf{w}_t-\eta \nabla F(\mathbf{w}_t)$, in which $\eta$ is the learning rate. However, $\nabla F(\mathbf{w}_t)$ is unknown in practice. We need to estimate it using the gradient vectors uploaded from clients. Therefore, at each iteration $t$, the master broadcasts parameter $\mathbf{w}_t$ to all clients, and then wait for the responses from them. Benign clients send the estimated gradient vectors back to the master, with respect to parameter $\mathbf{w}_t$. On the contrary, Byzantine clients send arbitrary vectors determined by the adversary. To be more precise, denote $\mathbf{X}_{it}$ as the vector received from client $i$ at the $t$-th iteration, then
\begin{eqnarray}
	\mathbf{X}_{it}&=&\left\{
	\begin{array}{ccc}
		\frac{1}{n_i}\sum_{j=1}^{n_i} \nabla f(\mathbf{w}_t, \mathbf{Z}_{ij}) &\text{if} & i\notin \mathcal{B}\\
		\star &\text{if} & i\in \mathcal{B},
	\end{array}
	\right.
	\label{eq:Xi}
\end{eqnarray}
in which $\star$ means arbitrary vector determined by the adversary. After received $\mathbf{X}_{it}$ for all $i=1,\ldots, m$, the master then updates the parameter $\mathbf{w}$ according to the following rule:
\begin{eqnarray}
	\mathbf{w}_{t+1}=\Pi_{\mathcal{W}}(\mathbf{w}_t-\eta g(\mathbf{w}_t)),
	\label{eq:update}
\end{eqnarray}
in which $\Pi_\mathcal{W}(\cdot)$ is an Euclidean projection operator, which ensures that the model parameter stays in $\mathcal{W}$. This operator is also used in \cite{yin2018byzantine,zhu2023byzantine}. $\eta$ is the learning rate. $g(\mathbf{w}_t)$ is the aggregator function, which estimates $\nabla F(\mathbf{w}_t)$ using $\mathbf{X}_{it}$, $i=1,\ldots, m$.
\begin{algorithm}[tb]
	\caption{Byzantine Robust Federated Learning}\label{alg:learning}
	\textbf{Input:} Master machine $S_0$, working machines $S_1,\ldots, S_m$\\
	\textbf{Parameter:} Initial weight parameter $\mathbf{w}_0\in \mathcal{W}$, learning rate $\eta$, total number of steps $T$\\
	\textbf{Output:}Estimated weight $\hat{\mathbf{w}}$
	\begin{algorithmic}
		\FOR{$t=0,1,\ldots, T-1$}
		\STATE \underline{\textit{Server}}: broadcast current parameter $\mathbf{w}_t$ to all clients;
		\FOR{$i=1,\ldots, m$ \textbf{in parallel}}
		\STATE \underline{\textit{Client} $i$}: compute local gradient $\mathbf{G}_i=(1/n_i)\sum_{j=1}^{n_i} \nabla f(\mathbf{w}_t, \mathbf{Z}_{ij})$;\\
		\IF{$i$ is benign}
		\STATE $\mathbf{X}_{it} = \mathbf{G}_i(\mathbf{w}_t)$;
		\ELSE
		\STATE $\mathbf{X}_{it}$ is an arbitrary $d$ dimensional vector determined by the adversary;
		\ENDIF
		\STATE send $\mathbf{X}_{it}$ to the server;
		\ENDFOR
		\STATE \underline{\textit{Server}}: Receive $\mathbf{X}_{it}$, $i=1,\ldots, m$ from each client;\\
		\STATE Calculate aggregated gradient $g(\mathbf{w}_t)$ using $\mathbf{X}_{it}$, $i=1,\ldots, m$;
		\STATE Update parameter $\mathbf{w}_{t+1} = \Pi_{\mathcal{W}}( \mathbf{w}_t-\eta g(\mathbf{w}_t))$;			
		\ENDFOR
		\RETURN $\hat{\mathbf{w}} = \mathbf{w}_T$
	\end{algorithmic}
\end{algorithm}

Now it remains to design the aggregator function $g(\mathbf{w}_t)$. Our idea is to minimize the Huber loss weighted by the sample sizes in each clients:
\begin{eqnarray}
	g(\mathbf{w}_t) = \underset{s}{\arg\min}\sum_{i=1}^m n_i\phi_i(\norm{\mathbf{s}-\mathbf{X}_{it}}),
	\label{eq:gw}
\end{eqnarray}
in which $\norm{\cdot}$ is the $\ell_2$ norm, and
\begin{eqnarray}
	\phi_i(u)=\left\{
	\begin{array}{ccc}
		\frac{1}{2}u^2 &\text{if} & |u|\leq T_i\\
		T_iu - \frac{1}{2} T_i^2&\text{if} & |u|>T_i
	\end{array}
	\right.
	\label{eq:phii}
\end{eqnarray}
is the Huber loss function. If clients have equal or nearly equal sample size $n_i$, then we can let $T_i$ to be the same for all $i$. Otherwise, we may use different $T_i$. With larger $n_i$, we let $T_i$ to be smaller. We refer detailed discussion on parameter selection to Section \ref{sec:iid} and \ref{sec:noniid}.

Now we explain the intuition of such design. We have two requirements for a good aggregator: consistency without attack, and robustness under attack. Minimizing $\ell_2$ loss corresponds to a simple averaging $g_{avg}(\mathbf{w}_t)=(1/m)\sum_{i=1}^{n_i}\mathbf{X}_{it}$, which is consistent without attack, but not robust. On the contrary, minimizing $\ell_1$ loss yields the geometric median of $\mathbf{X}_{it}$, $i=1,\ldots, m$, which is robust but not consistent even without attack. Therefore, we minimize Huber loss, which combines the advantages of these two methods. At the limit $T_i\rightarrow \infty$, $\phi_i$ becomes $\ell_2$ loss, then $g(\mathbf{w}_t)$ is just the sample mean $g_{avg}(\mathbf{w}_t)$. The opposite limit is $T_i\rightarrow 0$, under which $g(\mathbf{w}_t)$ is actually the weighted geometric median of $\mathbf{X}_{it}$, $i=1,\ldots, m$. Between these two extremes, we can set appropriate $T_i$ to achieve a good tradeoff between consistency and Byzantine robustness.

In the following sections, we provide a theoretical analysis of the performance of our method. Following previous works \cite{yin2018byzantine,zhu2023byzantine}, we discuss three cases separately, which are stated in Assumption \ref{ass:convexity}:

\begin{ass}\label{ass:convexity}
	Consider the following three types of global loss function $F$:
	
	(a) (Strong convex) $F$ is $\mu$-strong convex and $L$-smooth, and $\mathcal{W}$ is convex;
	
	(b) (General convex) $F$ is convex and $L$-smooth, with $\mathcal{W}=\{\mathbf{w}|\norm{\mathbf{w}-\mathbf{w}^*}\leq 2\norm{\mathbf{w}_0-\mathbf{w}^*} \}$;
	
	(c) (Non-convex) $F$ is $L$-smooth, and $\norm{\nabla F(\mathbf{w})}\leq M$ for all $\mathbf{w}\in \mathcal{W}$.
\end{ass}

For all these three cases, we make the following common assumption on the covering number, which ensures the uniform convergence of the aggregator function:
\begin{ass}\label{ass:cover}
	Assume that there exists constants  $C_W$, $r_D$, such that for all $r<r_D$, the $r$-covering of $\mathcal{W}$ is bounded by
	$N_c(r)\leq C_W/r^d$.
\end{ass} 

Before diving into the detailed analysis, we clarify the notations used in the remainder of this paper first.

\textbf{Notations.} Throughout this paper, we use the following notations: $a\lesssim b$ if there exists a constant $C$ such that $a\leq Cb$. $C$ may depend on $\mu$, $\sigma$, $L$ and $C_W$ in the assumptions. Conversely, $a\gtrsim b$ means $a\geq Cb$. Moreover, $a\sim b$ means there exists two constants $C_1$ and $C_2$ such that $C_1b\leq a\leq C_2 b$. Furthermore, $a=\tilde{O}(b)$ means $a\leq Cb\ln^k (N/\delta)$ for some constants $C$ and $k$. Denote $[m]=\{1,\ldots, m\}$ as the set of numbers from $1$ to $m$. $\epsilon$, $q$ are the ratio and the number of Byzantine clients, respectively, with $q=\epsilon m$. Finally, $\norm{\cdot}$ denotes $\ell_2$ norm.

\section{Theoretical Analysis for I.I.D Case}\label{sec:iid}
In this section, similar to most of previous works, we assume that all samples are i.i.d. We discuss two cases, depending on whether sample sizes are balanced in different clients.
\begin{ass}\label{ass:iid}
	$\mathbf{Z}_{ij}$ are i.i.d for all $i=1,\ldots, m$ and $j=1,\ldots, n_i$. For any $i$ and $j$,
	$\nabla f(\mathbf{w}, \mathbf{Z}_{ij})$ is sub-exponential with parameter $\sigma$, i.e.
	\begin{eqnarray}
		\underset{\|\mathbf{v}\|=1}{\sup}\mathbb{E}\left[e^{\lambda \mathbf{v}^T(\nabla f(\mathbf{w}, \mathbf{Z}_{ij}))-\nabla F(\mathbf{w})}\right]\leq e^{\frac{1}{2}\sigma^2 \lambda^2}, \text{ if } |\lambda|\leq \frac{1}{\sigma}
	\end{eqnarray}
	for any vector $\mathbf{v}$ with $\norm{\mathbf{v}} = 1$;	
\end{ass}
Assumption \ref{ass:iid} ensures that with high probability, using the true sample gradients, we are able to identify $\mathbf{w}^*$ defined in \eqref{eq:wopt}. Similar assumption was also made in \cite{chen2017distributed,cao2019distributed}.

\subsection{Balanced Data}
In this case, we assume that $n_i=N/m$ for all $i$, in which $N$ is the total number of training samples, and $m$ is the number of clients. Here we just denote $n$ as the number of samples per client, and the subscript $i$ is omitted. The analysis here can be simply generalized to the case in which $n_i$ are different but are in the same order, i.e. there exists two constants $c_1$ and $c_2$, such that $c_1n\leq n_i\leq c_2 n$. Since data sizes are balanced, we set equal thresholds in Huber losses, i.e. $T_i=T$ for all $i$.

We aim to obtain a bound of $\norm{\mathbf{w}_t-\mathbf{w}^*}$, the error of the estimation of global minimizer, that holds with probability at least $1-\delta$. In particular, the following theorem holds. For the sake of simplicity, we state the asymptotic version here, while the finite sample bounds and proofs are shown in the appendix. 
\begin{thm}\label{thm:simple}
	There exists two constants $C_1$ and $C_2$, if
	\begin{eqnarray}
		C_1\sqrt{\frac{d}{n}\ln \frac{N}{\delta}}\leq T\leq C_2\sqrt{\frac{d}{n}\ln \frac{N}{\delta}},
		\label{eq:trange}
	\end{eqnarray}
	then under Assumption \ref{ass:cover} and \ref{ass:iid}, with $|\mathcal{B}|=\epsilon m$ Byzantine clients, the following equations hold with probability at least $1-\delta$.
	
	(1) (Strong convex) Under Assumption 1(a), if $\eta\leq 1/L$,
	\begin{eqnarray}
		\norm{\mathbf{w}_t-\mathbf{w}^*}\leq (1-\rho)^t \norm{\mathbf{w}_0-\mathbf{w}^*}+\frac{2\Delta_A}{\mu},
	\end{eqnarray}
	in which $\rho = \eta\mu/2$;
	
	(2) (General convex) Under Assumption 1(b), with $\eta=1/L$, after $t_m=(L/\Delta_A)\norm{\mathbf{w}_0-\mathbf{w}^*}_2$ steps,
	\begin{eqnarray}
		F(\mathbf{w}_{t_m})-F(\mathbf{w}^*)\leq 16\norm{\mathbf{w}_0-\mathbf{w}^*}\Delta_A;
	\end{eqnarray}
	
	(3) (Non-convex) Under Assumption 1(c), with $\eta = 1/L$, after
	$t_m=(2L/\Delta_A^2)(F(\mathbf{w}_0) - F(\mathbf{w}^*))$ steps, we have
	\begin{eqnarray}
		\underset{t=0,1,\ldots, t_m}{\min}\norm{\nabla F(\mathbf{w}_t)}\leq \sqrt{2}\Delta_A,
	\end{eqnarray}
	in which 
	\begin{eqnarray}
		\Delta_A\lesssim \left(\frac{1}{\sqrt{1-2\epsilon}}\frac{\epsilon}{\sqrt{n}}+\frac{1}{\sqrt{N}}\right)\sqrt{d\ln \frac{N}{\delta}}.
		\label{eq:deltaa}
	\end{eqnarray}
\end{thm}

From \eqref{eq:trange}, the selection of parameter $T$ does not rely on the knowledge of $\epsilon$. Moreover, with fixed $d$, if $\epsilon$ is small, our error bound $\tilde{O}(\epsilon/\sqrt{n}+1/\sqrt{N})$ is nearly optimal, since it matches the information-theoretic minimax lower bound up to a logarithm factor \cite{yin2018byzantine}.  

\subsection{Unbalanced Data}
Now we discuss a more realistic setting, such that $n_i$ are different among clients. In this case, we design an adaptive selection rule of $T_i$. For a benign client $S_i$ with $n_i$ samples, from central limit theorem, the distance between the gradient vector send to the server decays roughly with the squared root of $n_i$, i.e. $\norm{\mathbf{X}_i(\mathbf{w}) - \nabla F(\mathbf{w})}\sim \sqrt{d/n_i}$ with high probability. Therefore, if $n_i$ is large and $S_i$ is benign, then $\mathbf{X}_i(\mathbf{w})$ should be close to most of other gradient vectors. If $\mathbf{X}_i(\mathbf{w})$ is far away from others, then client $S_i$ is highly likely to be Byzantine. However, if $n_i$ is small, even if $\mathbf{X}_i(\mathbf{w})$ is far away from others, we can not infer that $S_i$ is Byzantine, since the variance of $\mathbf{X}_i(\mathbf{w})$ is large even if $S_i$ is benign. With such intuition, we set $T_i$ to be smaller with large $n_i$, and vice versa.

To ensure the convergence, both the ratio of attacked clients and the ratio of the samples in attacked clients need to be small. Therefore, we slightly change the definition of $\epsilon$ as the maximum of the fraction of Byzantine clients, and the fraction of samples stored in Byzantine clients, i.e.
\begin{eqnarray}
	\epsilon = \max\left\{\frac{|\mathcal{B}|}{m}, \frac{\sum_{i\in \mathcal{B}}n_i}{N} \right\}.
	\label{eq:eps}
\end{eqnarray}
The theoretical results for unbalanced data is shown in Theorem \ref{thm:unbalanced}.
\begin{thm}\label{thm:unbalanced}
	Let
	\begin{eqnarray}
		T_i = T_0+\frac{M}{\sqrt{n_i}},
		\label{eq:adaptiveT}
	\end{eqnarray}
	with $M\sim \sigma\sqrt{d\ln (N/\delta)}$, and
	\begin{eqnarray}
		\epsilon\sigma \sqrt{\frac{md}{N}\ln \frac{N}{\delta}}\lesssim T_0\lesssim \sigma\sqrt{\frac{md}{N}\ln \frac{N}{\delta}}.
		\label{eq:t0range}
	\end{eqnarray}
	Then under Assumption \ref{ass:cover} and \ref{ass:iid}, the following equations hold for small $\epsilon$ with probability $1-\delta$.
	
	(1) (Strong convex) Under Assumption 1(a), if $\eta\leq 1/L$,
	\begin{eqnarray}
		\norm{\mathbf{w}_t-\mathbf{w}^*}\leq (1-\rho)^t \norm{\mathbf{w}_0-\mathbf{w}^*}+\frac{2\Delta_A}{\mu},
	\end{eqnarray}
	in which $\rho = \eta\mu/2$;
	
	(2) (General convex) Under Assumption 1(b), with $\eta=1/L$, after $t_m=(L/\Delta_B)\norm{\mathbf{w}_0-\mathbf{w}^*}_2$ steps,
	\begin{eqnarray}
		F(\mathbf{w}_{t_m})-F(\mathbf{w}^*)\leq 16\norm{\mathbf{w}_0-\mathbf{w}^*}\Delta_B;
	\end{eqnarray}
	
	(3) (Non-convex) Under Assumption 1(c), with $\eta = 1/L$, after
	$t_m=(2L/\Delta_B^2)(F(\mathbf{w}_0) - F(\mathbf{w}^*))$ steps, we have
	\begin{eqnarray}
		\underset{t=0,1,\ldots, t_m}{\min}\norm{\nabla F(\mathbf{w}_t)}\leq \sqrt{2}\Delta_B,
	\end{eqnarray}
	in which
	\begin{eqnarray}
		\Delta_B \lesssim \left(\frac{\epsilon}{1-2\epsilon}\sqrt{\frac{m}{N}}+ \frac{1}{\sqrt{N}}\right)\sqrt{d\ln \frac{N}{\delta}}.
		\label{eq:deltab}
	\end{eqnarray}
\end{thm}

Now we compare \eqref{eq:deltab} with \eqref{eq:deltaa}. If $\epsilon$ is close to $1/2$, then the error is larger with unbalanced data sizes than the balanced case. In particular, the factor $1/(1-2\epsilon)$ in \eqref{eq:deltab} is larger than the factor $1/\sqrt{1-2\epsilon}$ in \eqref{eq:deltaa}. However, if $\epsilon$ is small, we can neglect the factor $1/(1-2\epsilon)$, thus the only difference between \eqref{eq:deltab} and \eqref{eq:deltaa} is that the first term $\epsilon/ \sqrt{n}$ in the bracket in \eqref{eq:deltaa} is now replaced by $\epsilon \sqrt{m/N}$. Recall that if samples are evenly distributed, then $n=N/m$, thus these two bounds are actually of the same order. Therefore, we have shown a somewhat surprising result that the statistical error rate is not affected by the unbalanced allocation of training samples in clients.

\section{Theoretical Analysis for Non-I.I.D Case}\label{sec:noniid}
In this section, we assume that clients are heterogeneous. In particular, suppose that for any $\mathbf{w}$, $\mu_i(\mathbf{w})$, $i=1,\ldots, m$ are $m$ i.i.d random variables with $\mathbb{E}[\mu_i(\mathbf{w})] = F(\mathbf{w})$. Furthermore, assume that $\mathbf{Z}_{ij}$ for $j=1,\ldots, n_i$ are conditional independent given $\mu_i(\mathbf{w})$, and $\mathbb{E}[\nabla f(\mathbf{w}, \mathbf{Z}_{ij})|\mu_i(\mathbf{w})] = \mu_i(\mathbf{w})$. Now we replace Assumption \ref{ass:iid} with the following new assumption.
\begin{ass}\label{ass:niid}
	(a) $\mu_i(\mathbf{w})$ is sub-exponential with respect to $\nabla F(\mathbf{w})$ with parameter $\sigma_\mu$, i.e.
	\begin{eqnarray}
		\underset{\norm{\mathbf{v}} = 1}{\sup}\mathbb{E}\left[e^{\lambda \mathbf{v}^T (\mu_i(\mathbf{w}) - \nabla F(\mathbf{w}))}\right]\leq e^{\frac{1}{2}\sigma_\mu^2 \lambda^2}, \text{ if } |\lambda|\leq \frac{1}{\sigma_\mu};
	\end{eqnarray}
	
	(b) $\mathbf{Z}_{ij}$ is sub-exponential with respect to $\mu_i(\mathbf{w})$ with parameter $\sigma$, i.e.
	\begin{eqnarray}
		\underset{\norm{\mathbf{v}} = 1}{\sup}\mathbb{E}\left[e^{\lambda \mathbf{v}^T (\nabla f(\mathbf{w}, \mathbf{Z}_{ij}) - \mu_i(\mathbf{w}))}\right]\leq e^{\frac{1}{2}\sigma^2 \lambda^2}, \text{ if } |\lambda|\leq \frac{1}{\sigma},
	\end{eqnarray}
	in which $\mathbf{Z}_{ij}$, $j=1,\ldots, n_i$ are conditional i.i.d given $\mu_i(\mathbf{w})$.
\end{ass}
Assumption \ref{ass:niid}(a) allows heterogeneous data. However, $\mu_i(\mathbf{w})$ follows a sub-exponential distribution with parameter $\sigma_\mu$, thus the distance can not be too large. (b) requires that within each client, $\mathbf{Z}_{ij}$ is sub-exponential with respect to $\mu_i(\mathbf{w})$ for client $i$. At the limit of $\sigma_\mu \rightarrow 0$, Assumption \ref{ass:niid} reduces to Assumption \ref{ass:iid}. Some existing works assume that the clients are completely different, such as \cite{ghosh2019robust}, which groups clients into some clusters with inherently different properties. However, we assume that training samples are collected from sources that are only moderately different. The theoretical result is shown as following.

\begin{thm}\label{thm:niid}
	Let $T_i = T_{0a}+T_{0b}+M/\sqrt{n_i}$, with $M\sim \sigma \sqrt{d\ln (N/\delta)}$, and
	\begin{eqnarray}
		\epsilon\sigma \sqrt{(md/N)\ln (N/\delta)}\lesssim T_{0a}\lesssim \sigma \sqrt{(md/N)\ln (N/\delta)},	
	\end{eqnarray}
	and
	\begin{eqnarray}
		T_{0b}\sim \sigma\sqrt{(d/N)\ln (N/\delta)},
	\end{eqnarray}
	then under Assumption \ref{ass:cover} and \ref{ass:niid}, the following equations holds with probability at least $1-2\delta$.
	
	(1) (Strong convex) Under Assumption 1(a), with $\eta\leq 1/L$,
	\begin{eqnarray}
		\norm{\mathbf{w}_t-\mathbf{w}^*}\leq (1-\rho)^t \norm{\mathbf{w}_0-\mathbf{w}^*}+\frac{2\Delta_C}{\mu},
	\end{eqnarray}
	in which $\rho = \eta\mu/2$;
	
	(2) (General convex) Under Assumption 1(b), with $\eta=1/L$, after $t_m=(L/\Delta_C)\norm{\mathbf{w}_0-\mathbf{w}^*}_2$ steps,
	\begin{eqnarray}
		F(\mathbf{w}_{t_m})-F(\mathbf{w}^*)\leq 16\norm{\mathbf{w}_0-\mathbf{w}^*}\Delta_C;
	\end{eqnarray}
	
	(3) (Non-convex) Under Assumption 1(c), with $\eta = 1/L$, after
	$t_m=(2L/\Delta_C^2)(F(\mathbf{w}_0) - F(\mathbf{w}^*))$ steps, we have
	\begin{eqnarray}
		\underset{t=0,1,\ldots, t_m}{\min}\norm{\nabla F(\mathbf{w}_t)}\leq \sqrt{2}\Delta_C,
	\end{eqnarray}
	in which
	\begin{eqnarray}
		\Delta_C \hspace{-3mm}&\lesssim&\hspace{-3mm} \left(\frac{\epsilon}{1-2\epsilon}\sqrt{\frac{m}{N}}+ \frac{1}{\sqrt{N}}+\frac{\sigma_\mu\sqrt{\sum_{i=1}^m n_i^2}}{N}\right)\sqrt{d\ln \frac{N}{\delta}}+\epsilon\sigma_\mu d\ln \frac{N}{\delta}.
	\end{eqnarray}
\end{thm}

\section{Implementation}\label{sec:implement}

Our design follows the Weiszfeld's algorithm for geometric median \cite{weiszfeld:AOR:09}. Suppose there are $m$ vectors, $\mathbf{X}_1,\ldots, \mathbf{X}_m$. Define
\begin{eqnarray}
	\mathbf{c}=\underset{s}{\arg\min}\sum_{i=1}^m \phi_i(\norm{\mathbf{s}-\mathbf{X}_i})
\end{eqnarray}
as the center that minimizes the multi-dimensional Huber loss, in which $\phi_i$ is defined in \eqref{eq:phii}. Suppose that the algorithm starts from $\mathbf{c}_0$. Then the update rule is
\begin{eqnarray}
	\mathbf{c}_{k+1} = \frac{\sum_{i=1}^m \min\left\{1, \frac{T_i}{\norm{\mathbf{c}_k-\mathbf{X}_i}} \right\}\mathbf{X}_i}{\sum_{i=1}^m \min\left\{1, \frac{T_i}{\norm{\mathbf{c}_k-\mathbf{X}_i}} \right\}}.
	\label{eq:implement}
\end{eqnarray}
Our algorithm repeats \eqref{eq:implement} until convergence. Similar to Weiszfeld's algorithm, despite the fast convergence in almost all practical cases, it is not theoretically guaranteed in general. In section 8 in \cite{beck2015weiszfeld}, it is shown that under some assumptions, an estimate of geometric median with Weiszfeld's algorithm with error tolerance $\tau$ needs $O(1/\tau)$ steps. With minor modification, the analysis in \cite{beck2015weiszfeld} also holds for our new algorithm \eqref{eq:implement}. Moreover, from \eqref{eq:implement}, each step requires $O(md)$ time, thus the overall time complexity is $O(md/\tau)$.

In the future, it is possible to extend some recent works on geometric median, such as \cite{feldman2011unified,cohen2016geometric} to improve the update rule \eqref{eq:implement} for multi-dimensional Huber loss minimization.
\section{Comparison with Related Work}\label{sec:compare}
Now we compare our results with several existing popular approaches, including Krum \cite{blanchard2017machine}, geometric median-of-means \cite{chen2017distributed} (GMM), coordinate-wise median \cite{yin2018byzantine} (CWM), coordinate-wise trimmed mean \cite{yin2018byzantine} (CWTM), and recent methods based on high dimensional robust statistics \cite{shejwalkar2021manipulating,zhu2023byzantine} (HDRS). To begin with, we consider the i.i.d case. In particular, we compare the following five aspects listed as following:
\begin{itemize}
	\item Whether the method relies on precise knowledge of $\epsilon$;
	\item Under sub-exponential and strong convex assumption, with sufficiently large number of iterations $t$, whether the statistical error rate of $\norm{\hat{\mathbf{w}}-\mathbf{w}^*}$ is optimal (or nearly optimal up to a logarithm factor) in $\epsilon$. Here the optimal rate is $\epsilon/\sqrt{n}+\sqrt{d/N}$ \cite{Hopkins:COLT:19,yin2018byzantine};
	\item Whether the performance is good if $\epsilon$ is close to $1/2$. In particular, whether error blows up by a factor of no more than $1/\sqrt{1-2\epsilon}$, as is shown in \eqref{eq:deltaa};
	\item Whether the time complexity of aggregator is linear or nearly linear, i.e. $O(m)$ or $O(m\log m)$;
	\item Whether the error rate is optimal in dimensionality $d$;
\end{itemize}
The results are shown in Table \ref{tab:compare}, in which the five aspects mentioned above correspond to its five columns.

\begin{table}[h!]
	\centering
		\begin{tabular}{l|c|c|c|c|c}
			\hline
			Method & $\epsilon$-agnostic  & Optimal in $\epsilon$  &Good performance & linear time & Optimal in $d$\\
			&&&with $\epsilon$ near $1/2$ & complexity &\\
			\hline
			Krum &No &No  &No &No &No\\
			GMM &No  &No  &Yes &Yes &No\\
			CWM &Yes  &No  &Yes &Yes &No\\
			CWTM &No  &Yes  &No &Yes &No\\
			HDRS &No  &No  &No &No &Yes\\
			\textbf{Ours} & \textbf{Yes}  & \textbf{Yes}  &\textbf{Yes} & \textbf{Yes} &\textbf{No} \\
			\hline
		\end{tabular}
		\caption{Comparison of our method with existing aggregators under i.i.d assumption in five aspects. "$\epsilon$-agnostic" means that the method does not rely on precise knowledge of $\epsilon$.  "$\epsilon$-opt." refers to optimal dependence of error rate on $\epsilon$. Besides, "near $1/2$" means good performance with $\epsilon$ close to $1/2$. Moreover, "linear" stands for linear time complexity in $m$. Finally, "$d$-opt." means optimality of error rate in $d$.}
		\label{tab:compare}
	\end{table}
	
	For the first aspect, only coordinate-wise median and our method do not rely on precise knowledge of $\epsilon$. For other methods, $\epsilon$ significantly affects the parameter selection. For example, coordinate-wise trimmed mean need to set the trim threshold to be $\epsilon$ in both sides. However, it is quite unlikely to have exact knowledge of $\epsilon$ in practice. If an upper bound $\alpha$ is known, such that $\epsilon<\alpha$, then $\alpha$ can be used to set the parameter, but the accuracy will be sacrificed. Secondly, for optimal dependence in $\epsilon$, Krum is not guaranteed to converge globally. Coordinate-wise median is not optimal if $n\lesssim m$. Geometric median-of-mean only has a $\tilde{O}(\sqrt{\epsilon d/n})$ dependence. Thirdly, for the performance with $\epsilon$ close to $1/2$, coordinate-wise trimmed mean is not optimal. In particular, the error rate $\Delta_{CWTM}$, which plays the same role as $\Delta_A$ in Theorem \ref{thm:simple}, is
	\begin{eqnarray}
		\Delta_{CWTM} = \tilde{O}\left(\frac{d}{1-2\epsilon}\left(\frac{\epsilon}{\sqrt{n}}+\frac{1}{\sqrt{N}}\right)\right),
		\label{eq:cwtm}
	\end{eqnarray}  
	which blows up by a factor of $1/(1-2\epsilon)$, higher than $1/\sqrt{1-2\epsilon}$. Intuitively, with $\epsilon$ close to $1/2$, coordinate-wise trimmed mean is not efficient because it removes most of samples, thus somewhat wastes the data. For the fourth column in Table \ref{tab:compare}, the time complexity has been discussed in previous section.
	
	The only drawback of our method is that the dependence on $d$ is not optimal. In recent years, there are some new proposed methods for high dimensional robust mean estimation \cite{diakonikolas2016robust,diakonikolas2017being,diakonikolas2021list}. These methods can be used as the aggregator function in federated learning \cite{su2018securing,shejwalkar2021manipulating,zhu2023byzantine,zhao2023high}, so that the dependence on $d$ becomes optimal. However, the $\epsilon$ dependence is $\sqrt{\epsilon}$, which is optimal only if we merely require the covariance matrix to have bounded operator norms. Under a more restrictive sub-exponential assumption, the convergence can not be further improved. Moreover, the time complexity is much higher.
	
	One may wonder if it is possible to design an aggregator that satisfy all of the five properties listed in Table \ref{tab:compare}. However, \cite{Hopkins:COLT:19} has shown that for robust mean estimation problem, as long as $P\neq NP$, under sub-exponential assumption, polynomial time complexity, optimal dependence in $\epsilon$, and constant factor in $d$ are three goals that can not be achieved together. Since FL relies on robust mean estimation of gradient vectors as the aggregator function, we conjecture that it is impossible to satisfy all five properties in Table \ref{tab:compare}.
	
	Right now we have compared our results with existing methods under i.i.d assumptions. There are also some related work focusing on non-i.i.d cases, but the assumptions are crucially different. For example, \cite{Li:AAAI:19} proposed Robust Stochastic Aggregation (RSA). Instead of gradient aggregation, \cite{Li:AAAI:19} a new approach called model aggregation, which tries to reach a consensus between different models after training. \cite{pillutla2022robust} shows that a geometric median method can also be used here, named Robust Federated Aggregation (RFA). Furthermore, \cite{zuo2024byzantine} proposed Robust Average Gradient Algorithm (RAGA), which improves the geometric median aggregation and allows arbitrary round number for local updates. These methods are suitable for strongly heterogeneous cases, which means that the distribution of samples in different clients are significantly different. However, for slightly heterogeneous clients, the methods in \cite{Li:AAAI:19,pillutla2022robust,zuo2024byzantine} are suboptimal. In other words, under i.i.d limit, the error bounds are not comparable to existing methods designed for i.i.d cases. 
	\section{Numerical Results}\label{sec:numerical}
	This section shows numerical experiments. Despite the fact that our method has the advantage of not relying on knowledge of $\epsilon$, in this section, we just assume that all baseline methods know $\epsilon$ exactly, including Krum, GMM, CWM and CWTM mentioned in the previous section. Parameters of these methods are all optimally selected with the precise knowledge of $\epsilon$. If $\epsilon$ is unknown, then the advantage of our method should be larger than what is described in this section. The detailed algorithms and parameter selection rules are:
	
	\textbf{Krum \cite{blanchard2017machine}}: Let $k=m-q-2$, with $q$ being the number of Byzantine clients, then Krum picks $i^*$ as
	\begin{eqnarray}
		i^* = \underset{i\in [m]}{\arg\min}\sum_{j=1}^k \norm{\mathbf{X}_i-\mathbf{X}_i^{(j)}}^2,
		\label{eq:krum}
	\end{eqnarray}
	in which $\mathbf{X}_i^{(j)}$ is the $j$-th nearest neighbor of $\mathbf{X}_i$ among $\{\mathbf{X}_1,\ldots, \mathbf{X}_m \}$. Then
	\begin{eqnarray}
		g_{Krum}(\mathbf{w}) = X_{i*}.
		\label{eq:krum2}
	\end{eqnarray} 
	
	\textbf{Geometric median-of-mean (GMM) \cite{chen2017distributed}}. It randomly divides $m$ vectors from clients into $b$ batches, and then calculate geometric median of the mean gradients of these batches:
	\begin{eqnarray}
		g_{GMM}(\mathbf{w}) = \underset{s}{\arg\min}\sum_{j=1}^b\norm{\mathbf{s}-\frac{1}{|A_j|}\sum_{i\in A_j}\mathbf{X}_i}.
		\label{eq:gmm}
	\end{eqnarray}
	In our experiments, following \cite{chen2017distributed}, we set $b=2q+1$, in which $q$ is the number of Byzantine clients.
	
	\textbf{Coordinate-wise median (CWM)\cite{yin2018byzantine}}. This is an aggregation rule that calculates the element-wise median of vectors $\mathbf{X}_1,\ldots, \mathbf{X}_m$.
	
	\textbf{Coordinate-wise trimmed mean (CWTM)\cite{yin2018byzantine}}. We assume that $\epsilon$ is known. In each dimension, CWTM removes the largest and smallest $\epsilon$ fraction of values and then calculate the element-wise mean.

	Ideally, robustness against Byzantine failure needs to be tested using optimal attack strategies. However, the optimization problems are hard to solve. In this section, we use four attack strategies. One of them is a simple sign-flip attack, which flips the sign of gradient vectors. Other three are some approximate strategies that are tailored to specific aggregators, including Krum Attack (KA) and Trimmed Mean Attack (TMA) described in \cite{fang2020local} that are nearly optimal for Krum and CWTM, respectively. Moreover, for a fair comparison, we have designed an attack strategy for our new proposed method, called Huber Loss Minimization Attack (HLMA). Here we provide the details of all these attacks.
	
	Denote
	\begin{eqnarray}
		g_0(\mathbf{w}_t) = \frac{1}{m}\sum_{i=1}^m \mathbf{G}_i(\mathbf{w}_t),
	\end{eqnarray}
	with
	\begin{eqnarray}
		\mathbf{G}_i(\mathbf{w}_t)=\frac{1}{n_i}\sum_{j=1}^{n_i}\nabla f(\mathbf{w}_t, \mathbf{Z}_{ij}).
	\end{eqnarray}
	Let $\mathbf{s}_t=\sign(g_0(\mathbf{w}_t))$ to be the element-wise sign of $g_0(\mathbf{w}_t)$. We solve the following optimization problem:
	\begin{eqnarray}
		\underset{\pi_A}{\max} \mathbf{s}^T (g_0(\mathbf{w}_t) - g(\mathbf{w}_t)),
	\end{eqnarray}
	in which $\pi_A$ is the attack strategy. Such formulation follows \cite{fang2020local}. The intuition is to make the aggregated gradient should be as far away from the ground truth $\nabla F(\mathbf{w}_t)$ (approximated by $g_0(\mathbf{w}_t))$) as possible. Moreover, in each coordinate, it would be better if $g(\mathbf{w}_t)$ is moved towards the opposite direction to $g_0(\mathbf{w}_t)$, such that the model updates along a wrong direction. Our design of Krum attack and Trimmed mean attack follows \cite{fang2020local}.
	
	\textbf{Krum Attack (KA)}. KA maximizes $\lambda$, such that with all attacked samples placed at $g_0(\mathbf{w}_t)-\lambda \mathbf{s}$, Krum selects one of these attacked samples as the aggregated gradient according to \eqref{eq:krum}. Practically, we set $\lambda = 1$ first, and then repeat with $\lambda\leftarrow \lambda / 2$ until Krum selects one of attacked samples, or $\lambda<10^{-4}$.
	
	\textbf{Trimmed Mean Attack (TMA)}. Denote $X_{it,j}$ as the $j$-th element of $\mathbf{X}_{ij}$, and
	\begin{eqnarray}
		G_{ij}(\mathbf{w}_t) = \frac{1}{n_i}\sum_{j=1}^{n_i}\partial_j f(\mathbf{w}_t, \mathbf{Z}_{ij}).
	\end{eqnarray}
	Then the attack strategy is
	\begin{eqnarray}
		X_{it, j}=\left\{
		\begin{array}{ccc}
			\underset{i\in [m]}{\max} G_{ij}(\mathbf{w}_t) &\text{if} & s=-1\\
			\underset{i\in [m]}{\min} G_{ij}(\mathbf{w}_t) & \text{if} & s=1.
		\end{array}
		\right.
	\end{eqnarray}
	It has been discussed in \cite{fang2020local} that the TMA is also optimal for coordinate-wise median.
	
	\textbf{Huber Loss Minimization Attack (HLMA).} We let
	\begin{eqnarray}
		X_{it, j}=\left\{
		\begin{array}{ccc}
			G_{ij} (\mathbf{w}_t)+\frac{T}{\sqrt{d}} &\text{if} & s=-1\\
			G_{ij}(\mathbf{w}_t)-\frac{T}{\sqrt{d}} &\text{if}& s=1.
		\end{array}
		\right.
	\end{eqnarray}
	
	In Section \ref{sec:synthesized} and \ref{sec:real}, we show simple experiments on some synthesized data and MNIST data, respectively. In these two sections, samples are i.i.d and data allocation is balanced.
	
	\subsection{Synthesized Data}\label{sec:synthesized}
	
	To begin with, we run experiments with distributed linear regression. The model is
	\begin{eqnarray}
		V_j = \langle \mathbf{U}_j, \mathbf{w}^*\rangle + W_j,
		\label{eq:model}
	\end{eqnarray}
	in which $\mathbf{U}_j, \mathbf{w}^*\in \mathbb{R}^d$, with $\mathbf{w}^*$ being the true parameter, and $W_j$ is the noise following standard normal distribution. In this experiment, we set $d=50$. Firstly, we generate all elements of $\mathbf{w}^*$ from distribution $\mathcal{N}(0,1)$. We then obtain $N=10,000$ samples $(\mathbf{U}_j, V_j)$, $j=1,\ldots, N$. These samples are evenly divided into $m=500$ clients. For baseline methods including Krum, GMM, CWTM, the parameters are all set optimally, according to the statements at the beginning of Section \ref{sec:numerical}. Our new Huber loss minimization approach uses $T=1$ for all clients. We run $200$ iterations in total, with learning rate $\eta = 0.02$. The results are shown in Figure \ref{fig:synthesized} for all four attack strategies mentioned above, in which we plot the square root of the $\ell_2$ regression loss against the number of iterations.
	\begin{figure}[h!]
		\centering	
		\begin{subfigure}{0.24\linewidth}
			\includegraphics[width=1.05\textwidth,height=0.9\textwidth]{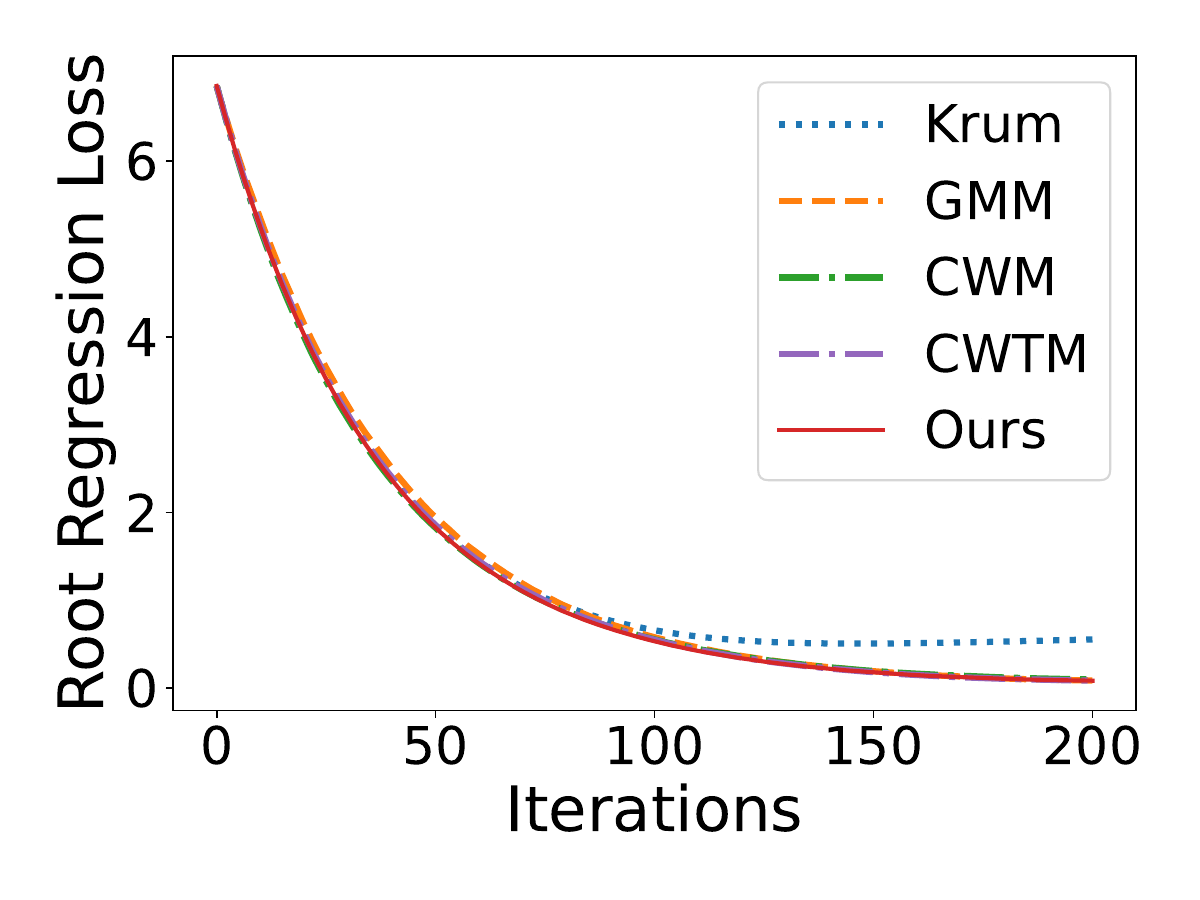}
			\caption{Sign-flip with $\epsilon = 0.2$.}
		\end{subfigure}	
		\begin{subfigure}{0.24\linewidth}
			\includegraphics[width=1.05\textwidth,height=0.9\textwidth]{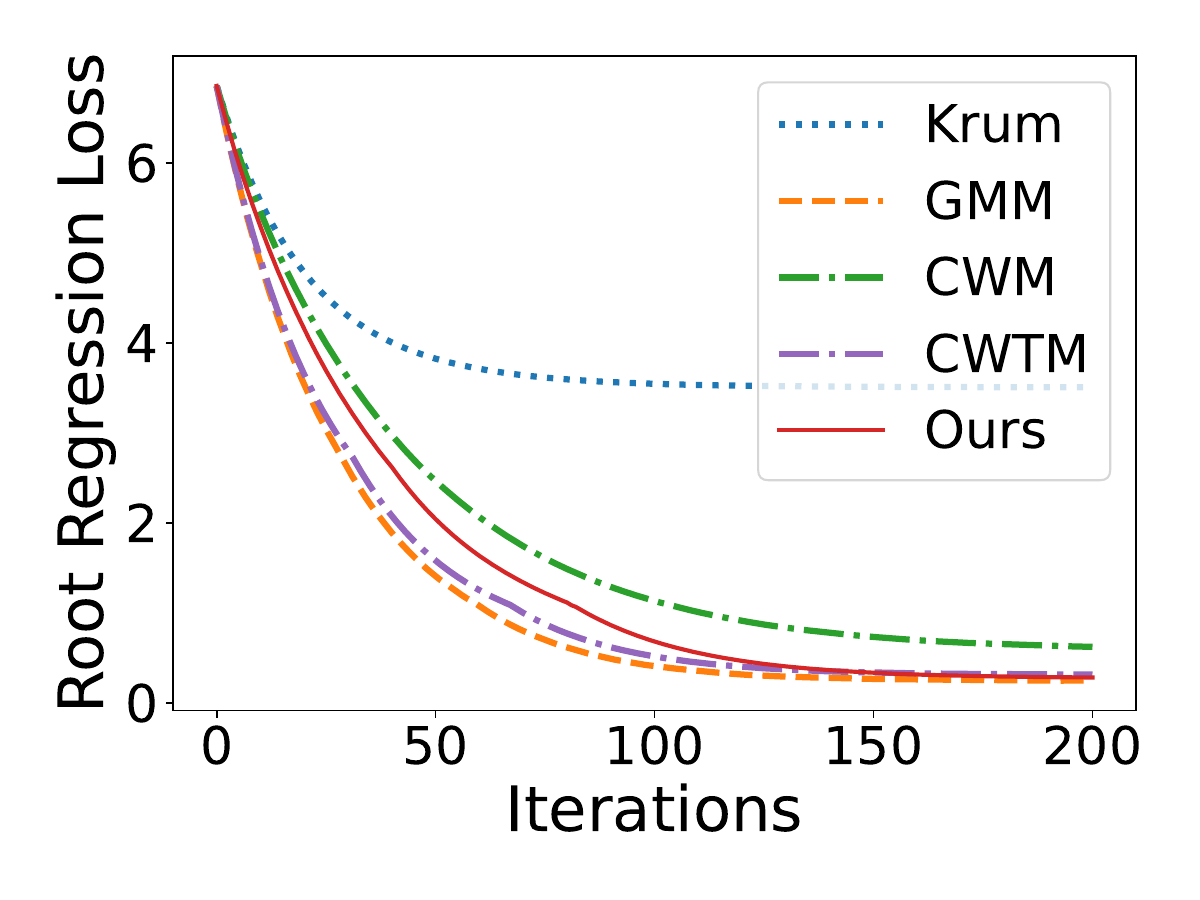}
			\caption{KA with $\epsilon = 0.2$.}
		\end{subfigure}
		\begin{subfigure}{0.24\linewidth}
			\includegraphics[width=1.05\textwidth,height=0.9\textwidth]{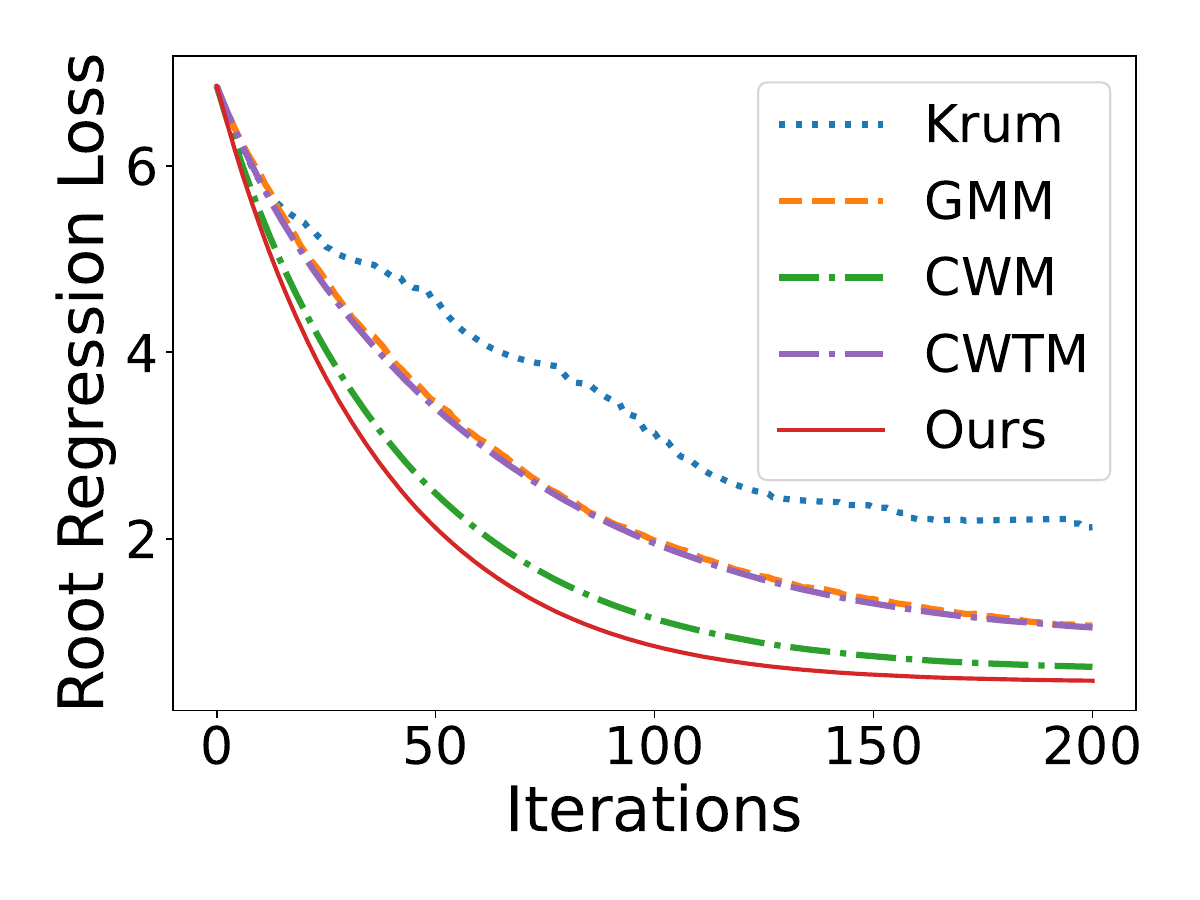}
			\caption{TMA with $\epsilon = 0.2$.}
		\end{subfigure}
		\begin{subfigure}{0.24\linewidth}
			\includegraphics[width=1.05\textwidth,height=0.9\textwidth]{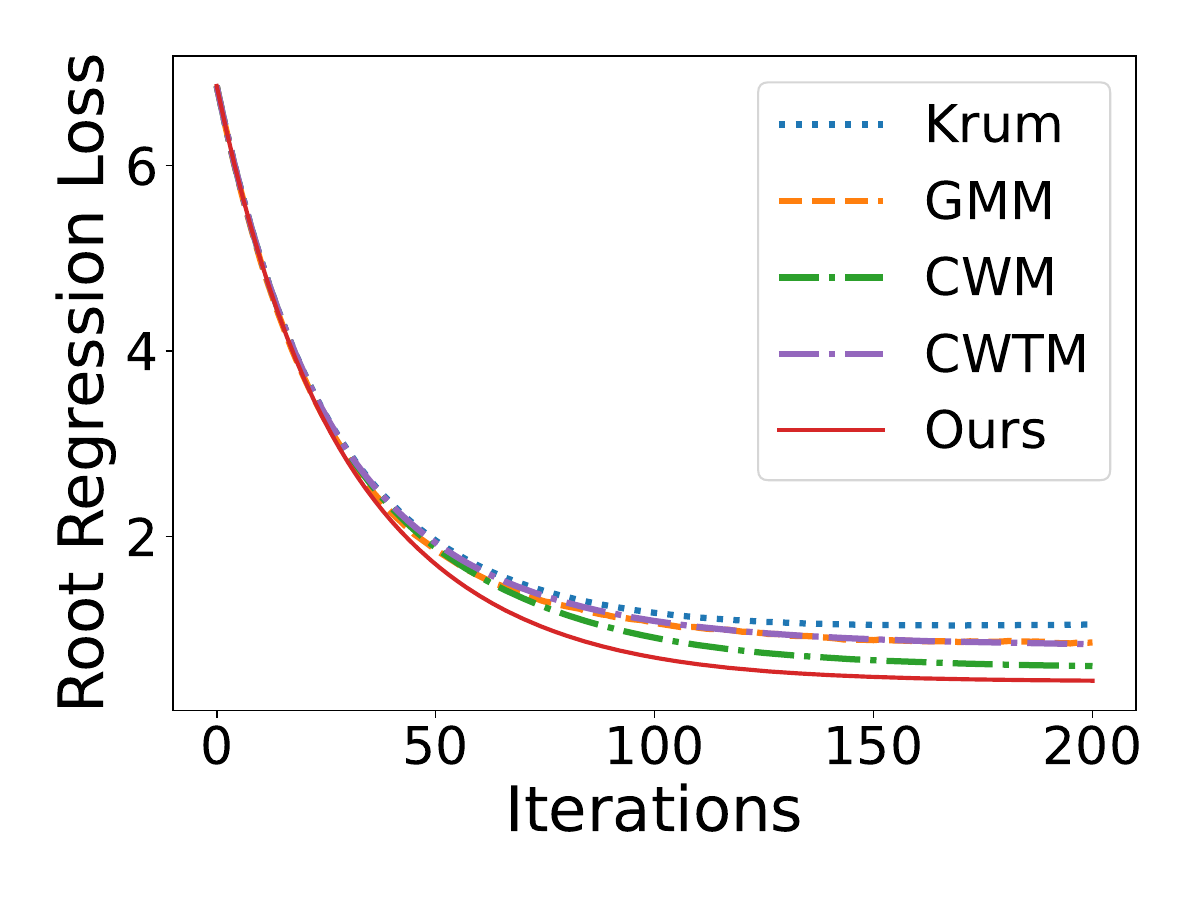}
			\caption{HLMA with $\epsilon = 0.2$.}
		\end{subfigure}	
		\caption{Comparison of our new method and several baselines against Krum Attack and Trimmed Mean Attack for synthesized data with $\epsilon=0.2$.}\label{fig:synthesized}
	\end{figure}
	
	From Figure \ref{fig:synthesized}, it can be observed that our new method (red solid curve) works well for all four types of attacks, even with HLMA that is specifically designed for our method. On the contrary, from Figure \ref{fig:synthesized} (b), Krum aggregator (blue dotted curve) fails under KA that is tailored to Krum. Moreover, (c) shows that under TMA which is designed for coordinate-wise trimmed mean, GMM and CWTM (dashed curves with orange and purple colors, respectively) perform significantly worse than our method. CWM appears to be robust, but the performance is not as good as our approach in general.

	\subsection{Real data}\label{sec:real}
	Now we use MNIST dataset \cite{lecun1998mnist}, which has $60,000$ images of handwritten digits for training, and $12,000$ images for testing. The sizes of these images are $28\times 28$. In this experiment, we use a neural network with one hidden layer between input and output. The size of hidden layer is $32$. Training samples are evenly allocated into $m=500$ clients. We set $T=0.2$ here. The results with $\epsilon=0.2$ is shown in Figure \ref{fig:mnist}, in which we plot the curve of accuracy on the test data against the number of iterations.
	
	\begin{figure}[h!]
		\centering	
		\begin{subfigure}{0.24\linewidth}
			\includegraphics[width=1.05\textwidth,height=0.9\textwidth]{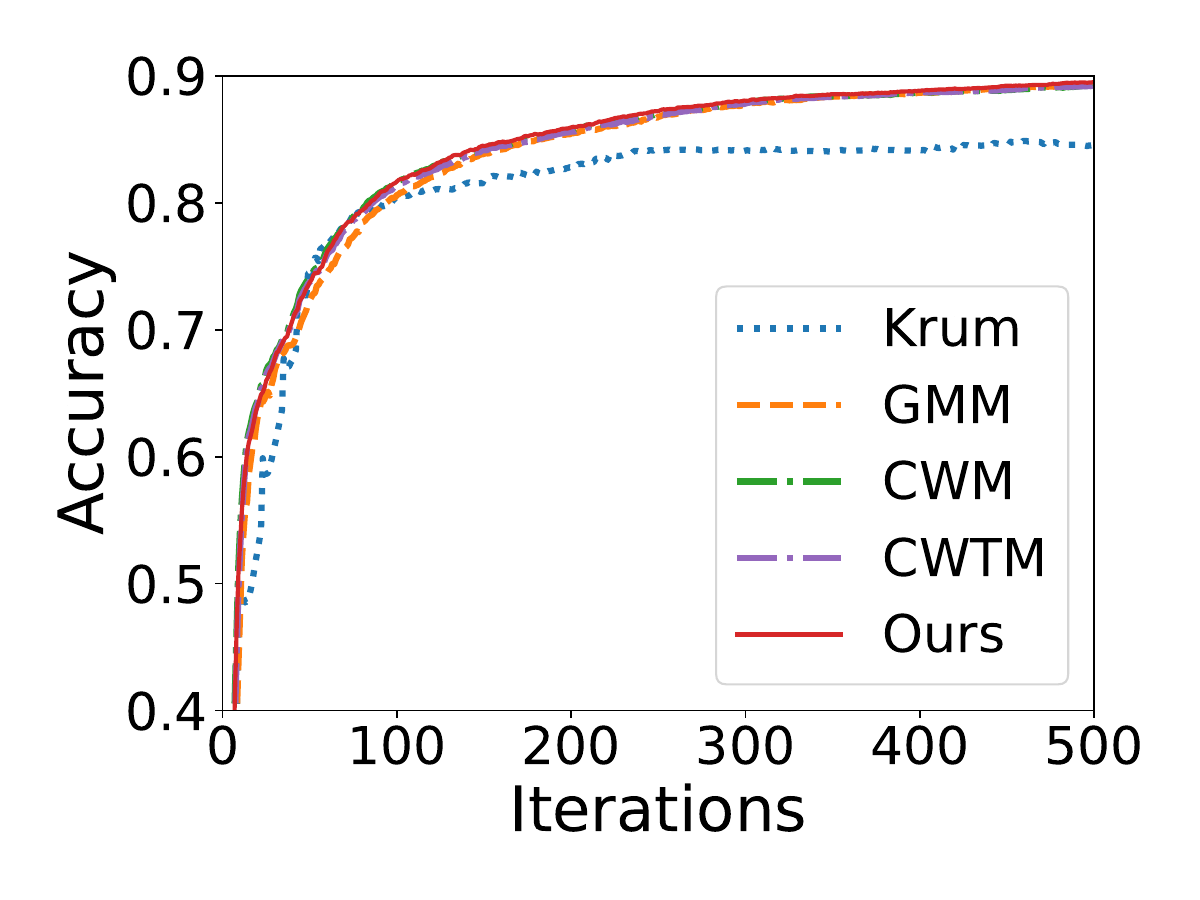}
			\caption{SignFlip with $\epsilon = 0.2$.}
		\end{subfigure}	
		\begin{subfigure}{0.24\linewidth}
			\includegraphics[width=1.05\textwidth,height=0.9\textwidth]{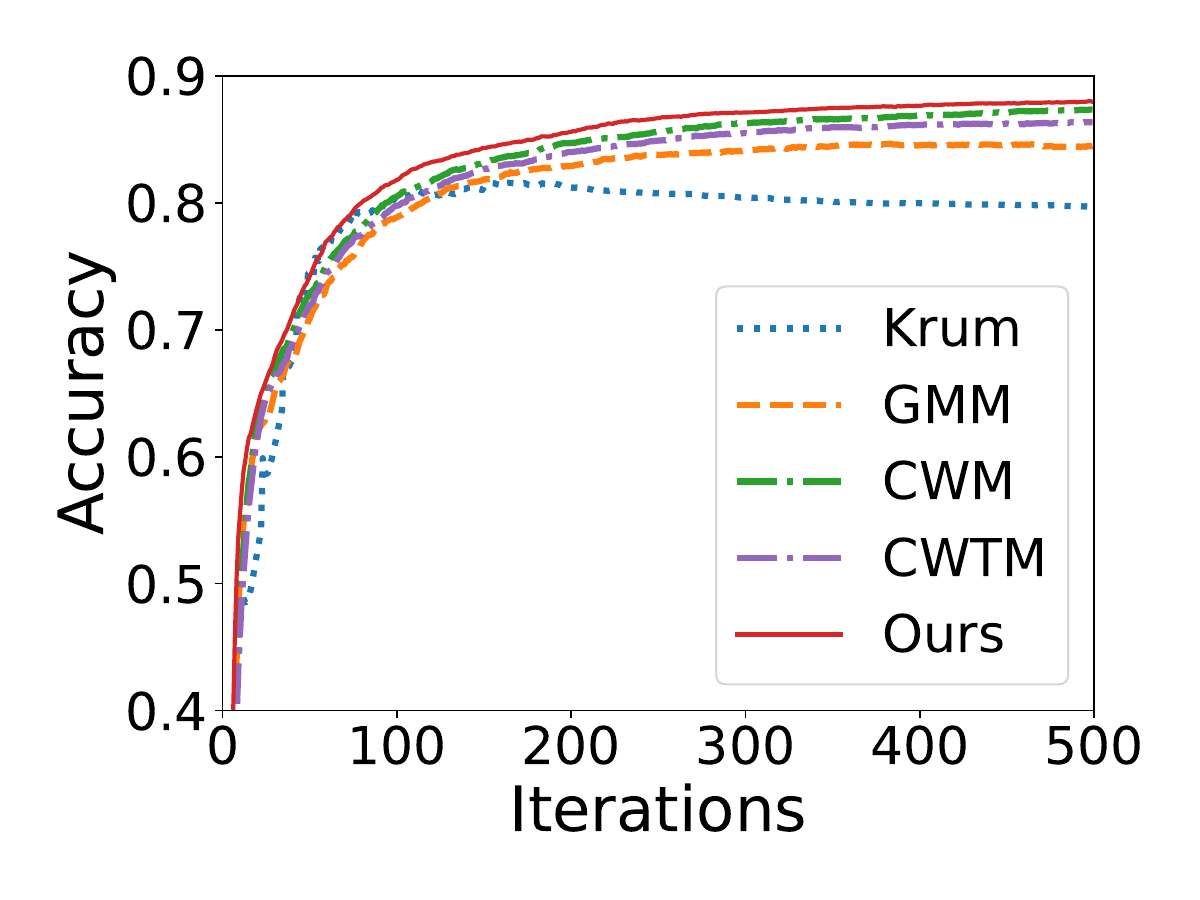}
			\caption{KA with $\epsilon = 0.2$.}
		\end{subfigure}
		\begin{subfigure}{0.24\linewidth}
			\includegraphics[width=1.05\textwidth,height=0.9\textwidth]{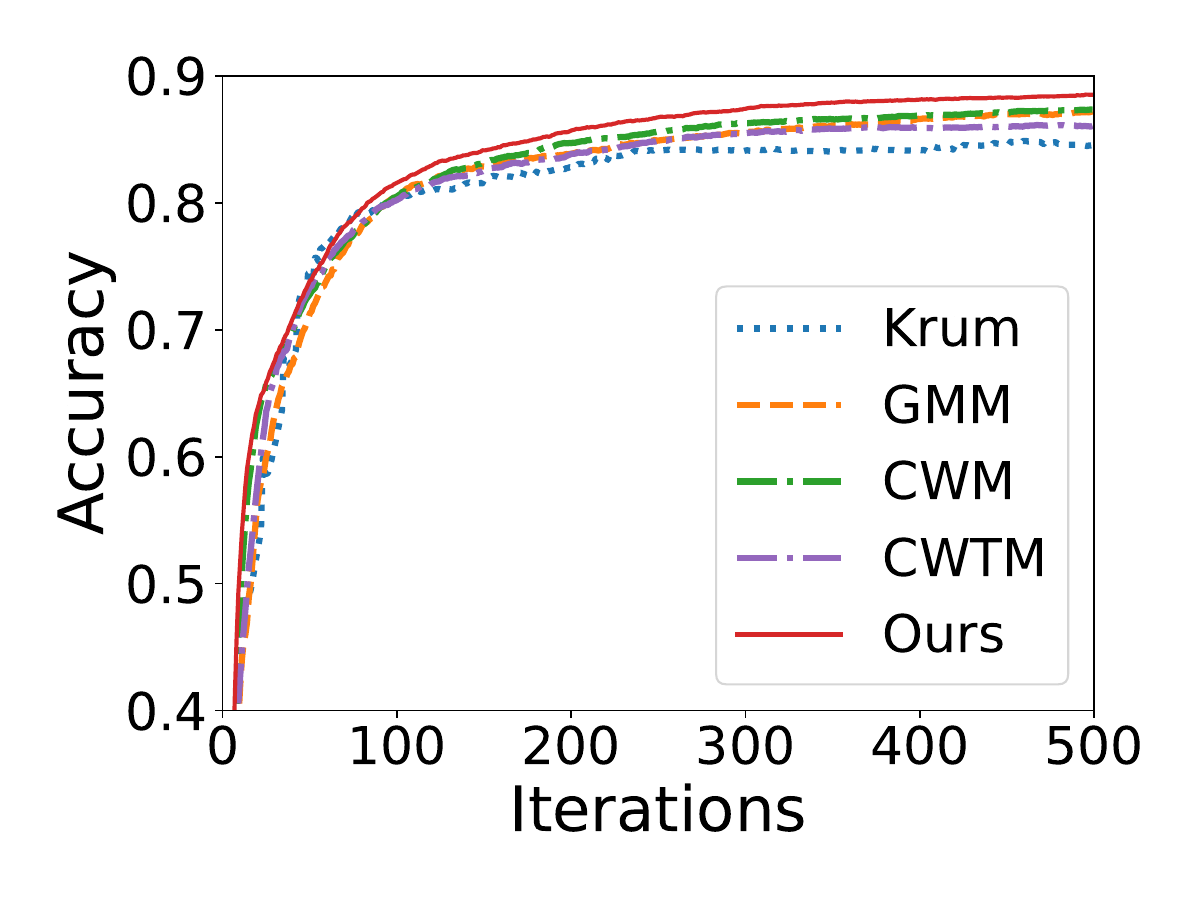}
			\caption{TMA with $\epsilon = 0.2$.}
		\end{subfigure}
		\begin{subfigure}{0.24\linewidth}
			\includegraphics[width=1.05\textwidth,height=0.9\textwidth]{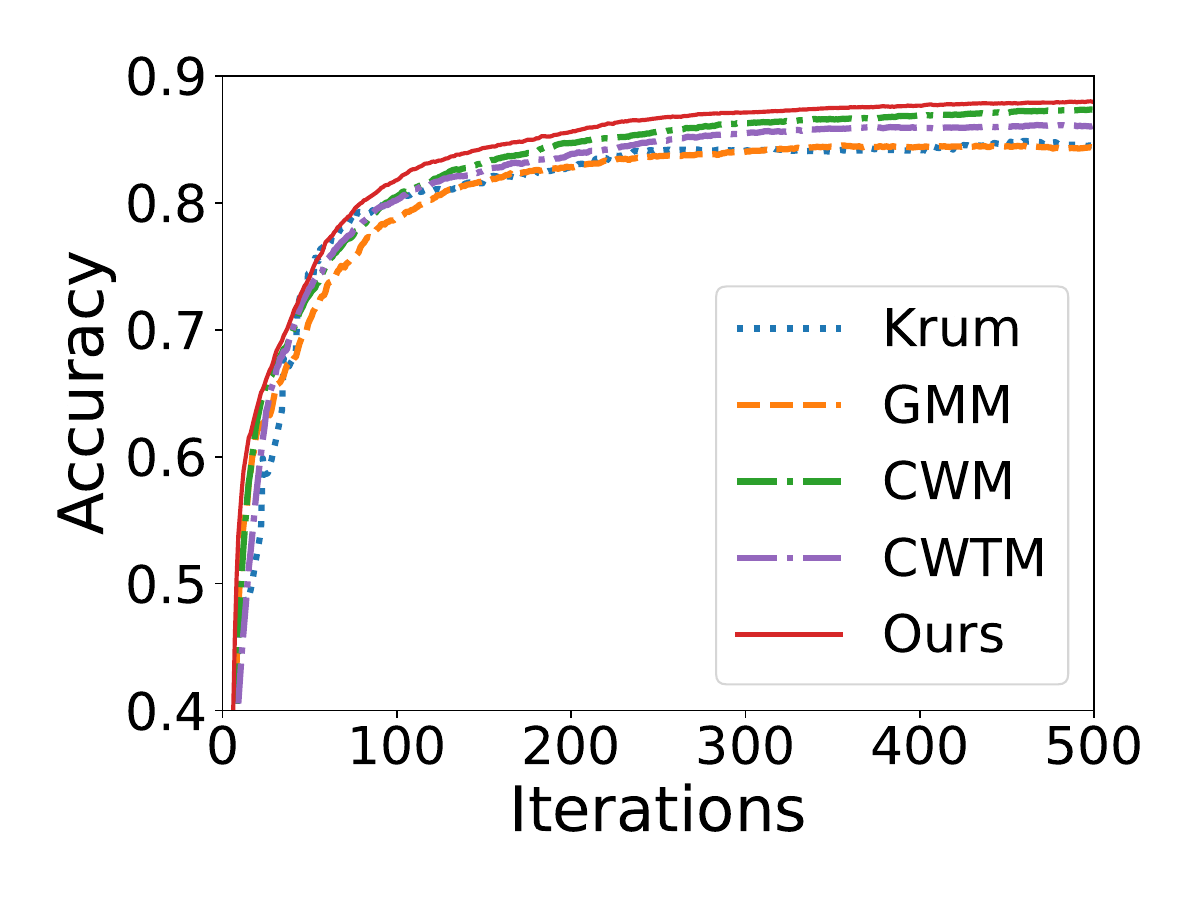}
			\caption{HLMA with $\epsilon = 0.2$.}
		\end{subfigure}	
		\caption{Comparison of our new method and several baselines against sign-flip, KA, TMA and HLMA for MNIST data, with $\epsilon=0.2$.}\label{fig:mnist}
	\end{figure}
	Similar to synthesized data, experiments on MNIST show that our new method outperforms other methods. Krum is still highly susceptible to KA. CWM and CWTM appear to be only slightly worse than our method.
	
	Now we increase $\epsilon$ to $0.4$. Other settings remain unchanged. This experiment shows the performance when $\epsilon$ is close to $1/2$. The results are shown in Figure \ref{fig:largeeps}.
	\begin{figure}[h!]
		\centering	
		\begin{subfigure}{0.24\linewidth}
			\includegraphics[width=1.05\textwidth,height=0.9\textwidth]{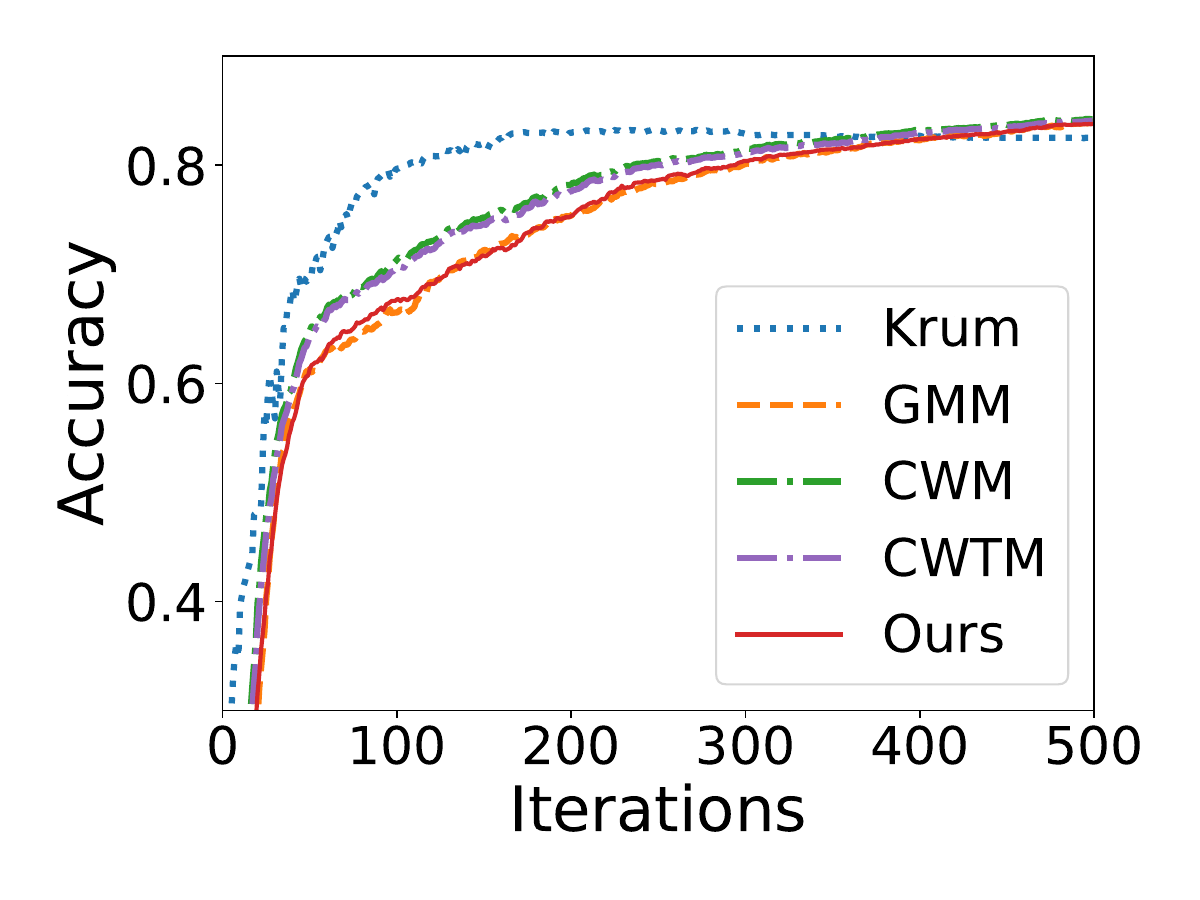}
			\caption{SignFlip with $\epsilon = 0.4$.}
		\end{subfigure}	
		\begin{subfigure}{0.24\linewidth}
			\includegraphics[width=1.05\textwidth,height=0.9\textwidth]{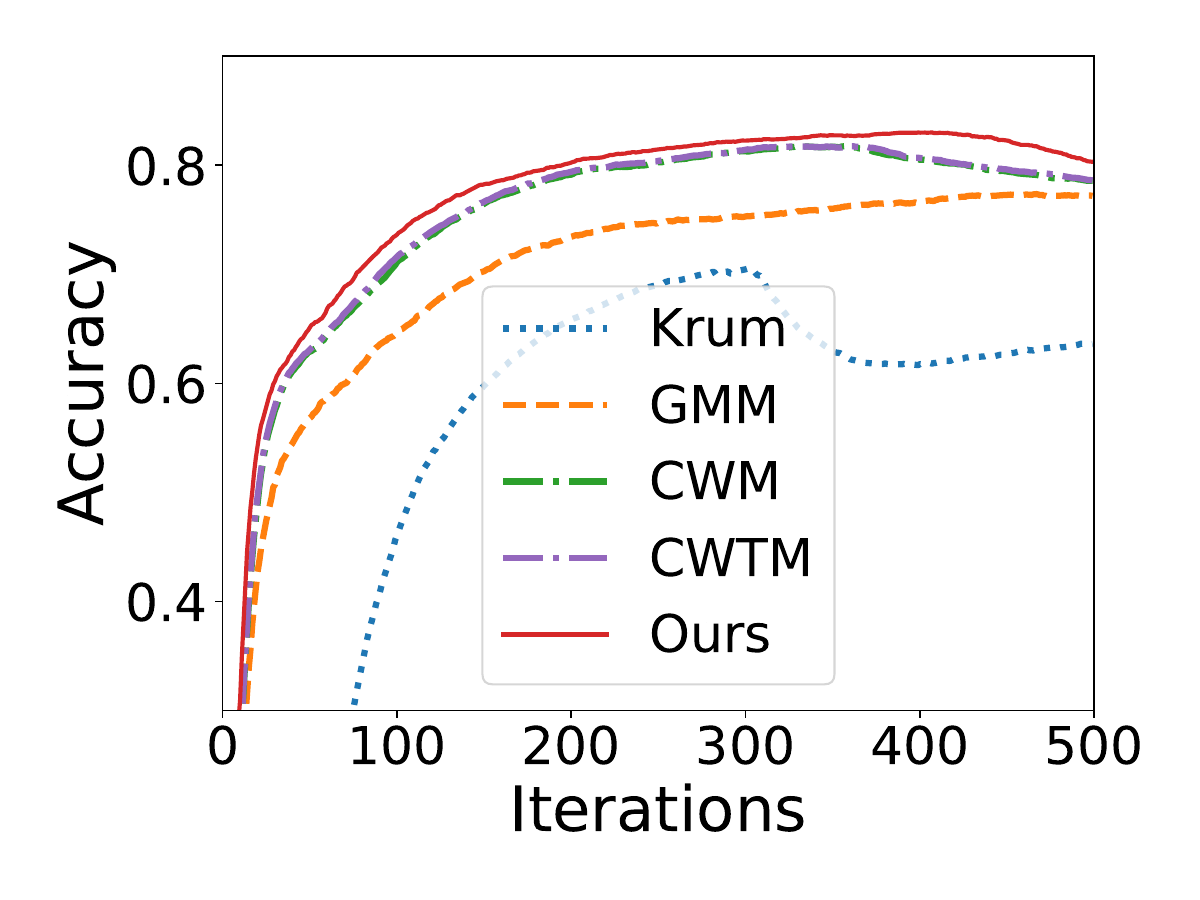}
			\caption{KA with $\epsilon = 0.4$.}
		\end{subfigure}
		\begin{subfigure}{0.24\linewidth}
			\includegraphics[width=1.05\textwidth,height=0.9\textwidth]{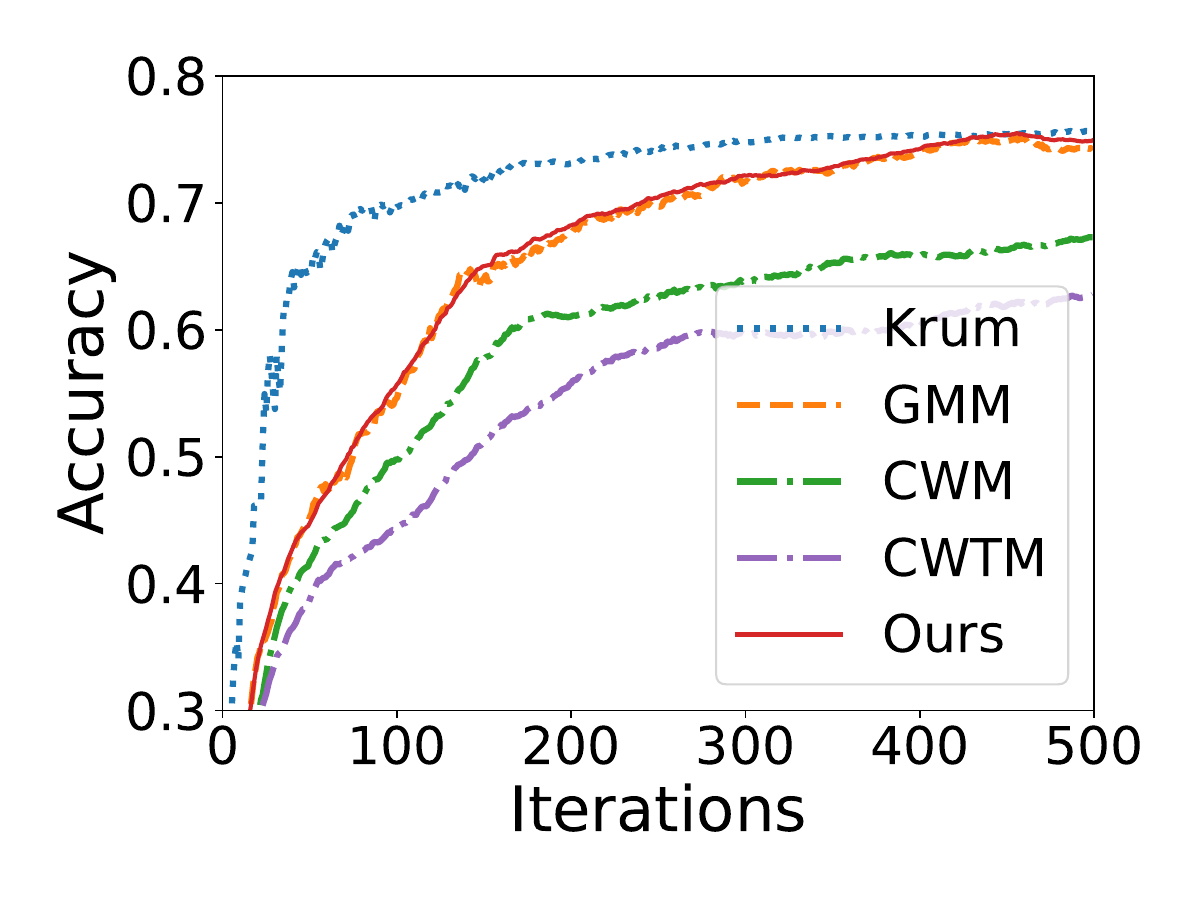}
			\caption{TMA with $\epsilon = 0.4$.}
		\end{subfigure}
		\begin{subfigure}{0.24\linewidth}
			\includegraphics[width=1.05\textwidth,height=0.9\textwidth]{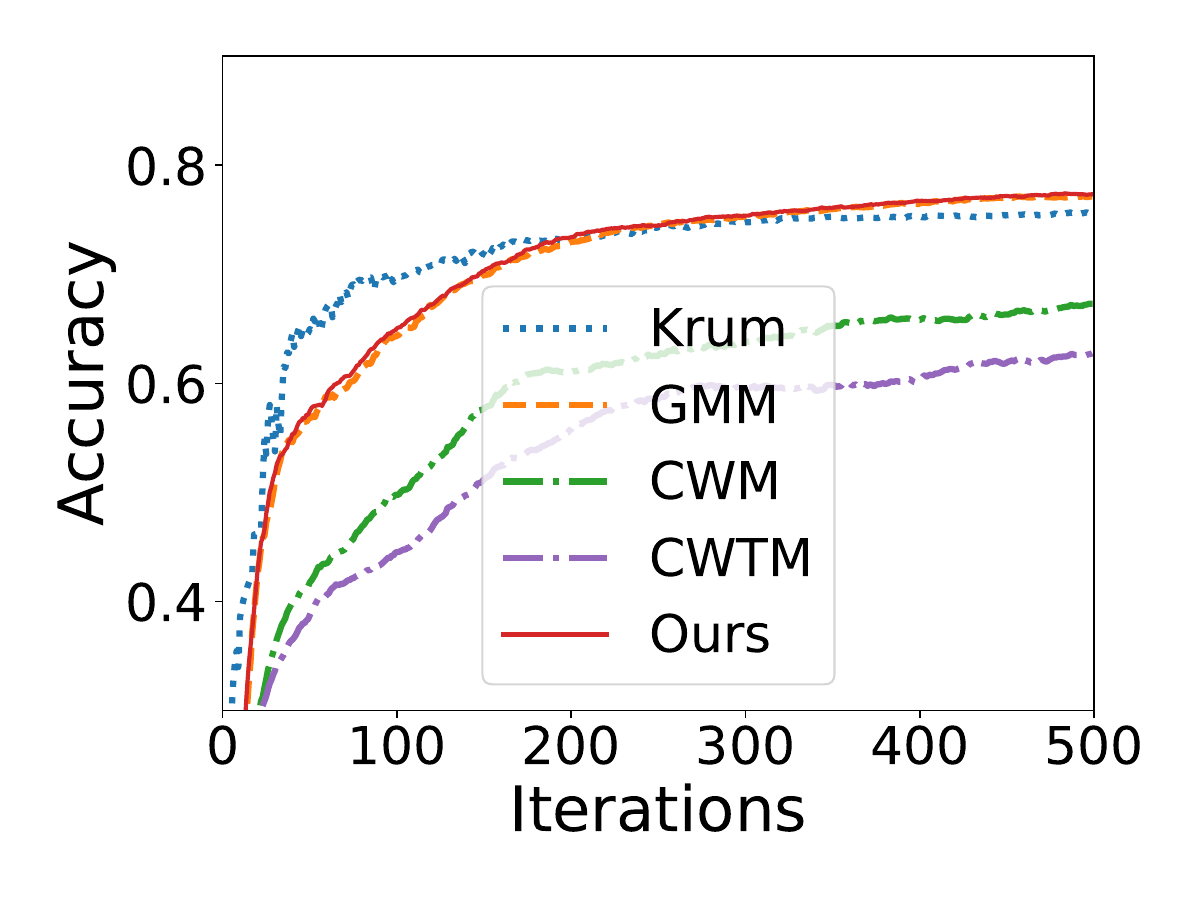}
			\caption{HLMA with $\epsilon = 0.4$.}
		\end{subfigure}	
		\caption{Comparison of our new method and several baselines against sign-flip, KA, TMA and HLMA for MNIST data, with $\epsilon=0.4$.}\label{fig:largeeps}
	\end{figure}
	
	Our method still exhibits desirable performance, even under HLMA designed specifically for ourselves. Krum works relatively well under sign-flip, TMA and HLMA, but fails under KA. The performance of GMM is relatively well for $\epsilon = 0.4$. As has been discussed earlier, the drawback of GMM is the suboptimal statistical rate on $\epsilon$, which is serious for small $\epsilon$. However, when $\epsilon$ is close to $1/2$, this disadvantage becomes much less obvious. Moreover, under TMA and HLMA, the performance of coordinate-wise methods, including CWM and CWTM become significantly worse, especially the latter. This phenomenon agrees with the theoretical analysis, since in section \ref{sec:compare}, we have discussed that CWTM has suboptimal rate with $\epsilon$ close to $1/2$.

	\subsection{Unbalanced sample allocation}
	We then evaluate the adaptive selection rule \eqref{eq:adaptiveT} for unbalanced data. In this experiment, samples are still allocated in different clients randomly, but the sample sizes in different clients are no longer equal. 
	
	In this experiment, we use the following allocation rule.
	
	(1) Randomly shuffle training samples. Denote the indices of samples after shuffling as $0,1,\ldots, N-1$;
	
	(2) Generate $m-1$ indices randomly from $\{0,\ldots, N-1 \}$ without replacement;
	
	(3) Sort these $m-1$ indices in ascending order. Denote the results as $b_1,\ldots, b_{m-1}$, such that $b_1<b_2<\ldots<b_{m-1}$;
	
	(4) Let $b_0=0$, $b_m=N$. Allocate samples with $b_i\leq j\leq b_{i+1}$ to client $i$.
	
	It can be shown that with increase of $N$, the ratio of samples in each client follows Beta distribution $\mathbb{B}(1,m)$.
	
	We use $T_i=2/\sqrt{n_i}$ as the threshold of each clients. Other settings remain the same as previous experiments. For simplicity, we only show the result with HLMA in Figure \ref{fig:unbalance}.
	\begin{figure}[h!]
		\centering	
		\begin{subfigure}{0.33\linewidth}
			\includegraphics[width=1.05\textwidth,height=0.9\textwidth]{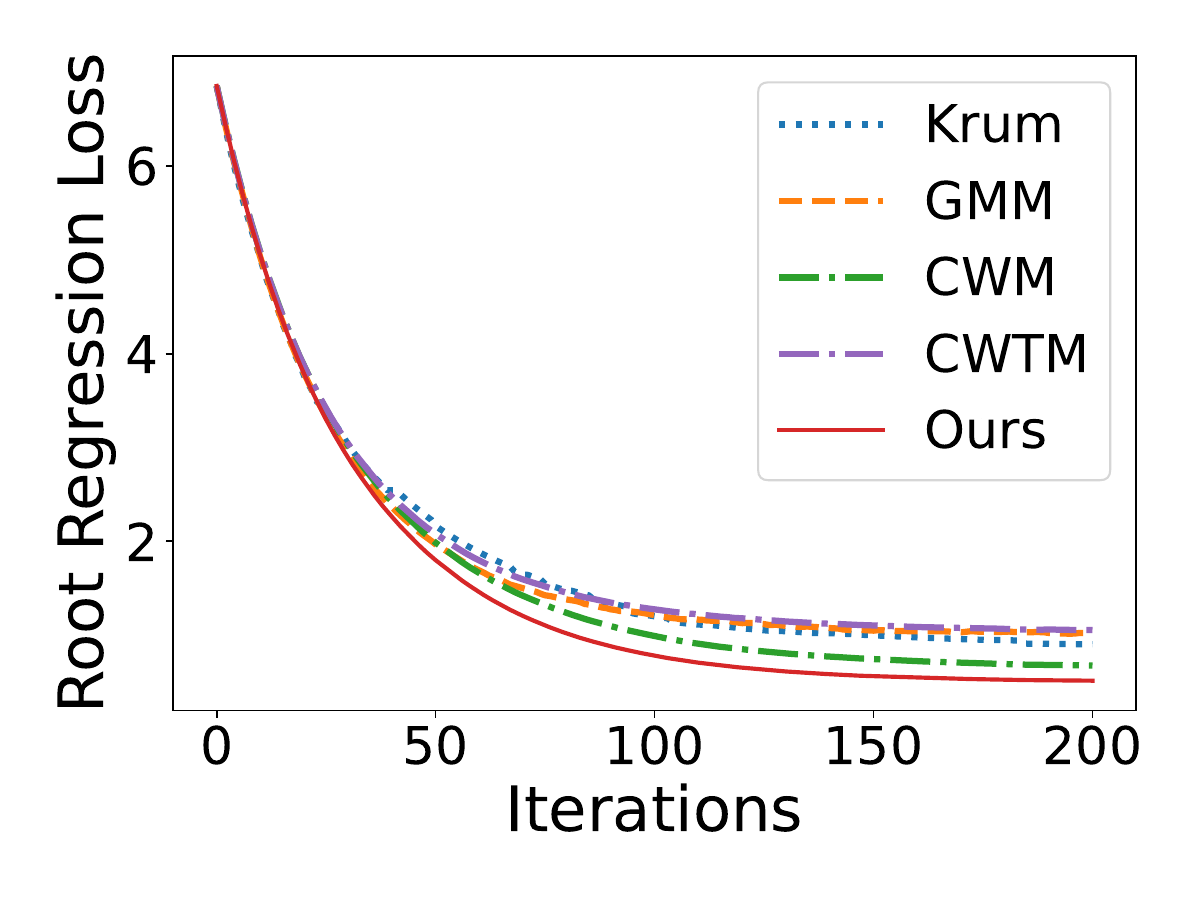}
			\caption{Synthesized data.}
		\end{subfigure}	
		\begin{subfigure}{0.33\linewidth}
			\includegraphics[width=1.05\textwidth,height=0.9\textwidth]{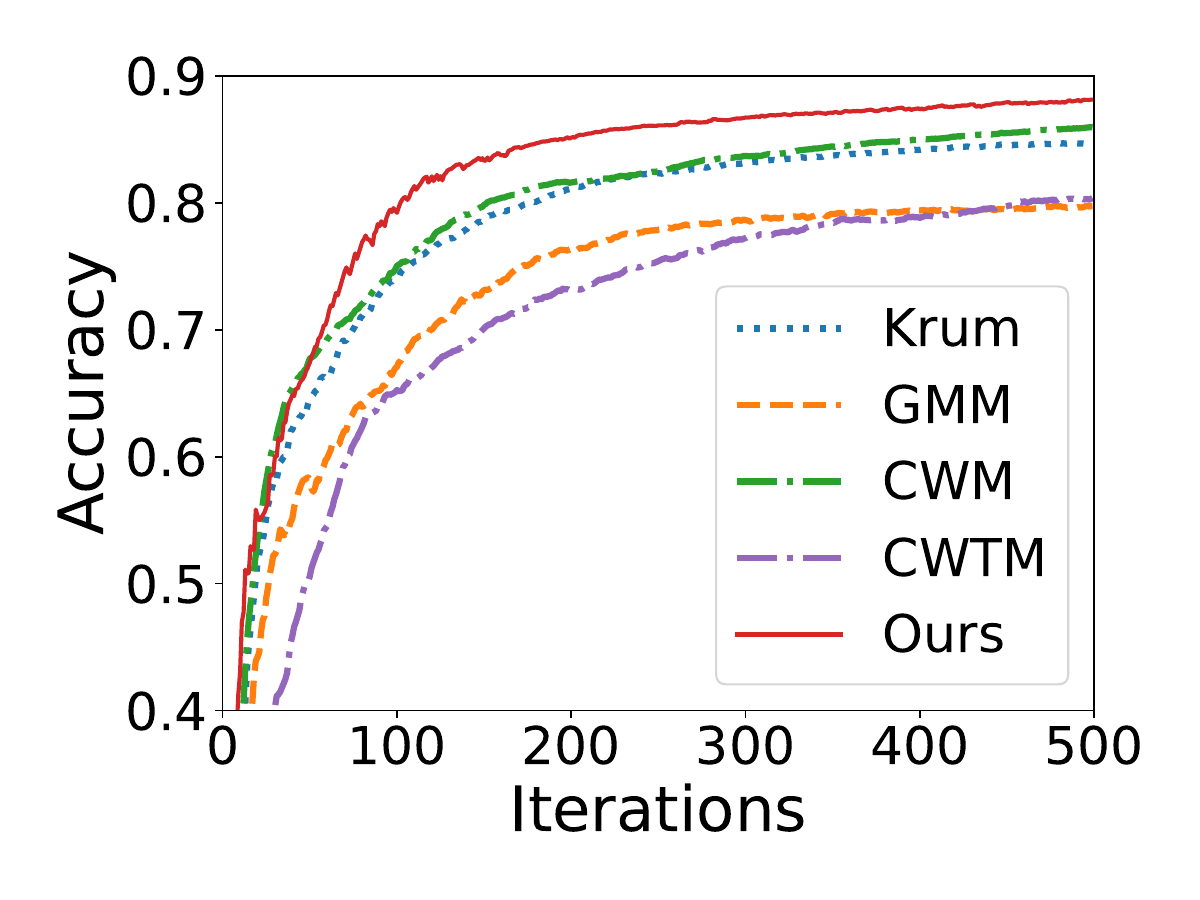}
			\caption{MNIST data.}
		\end{subfigure}	
		\caption{Experiments on unbalanced data for HLMA, with $\epsilon=0.2$.}\label{fig:unbalance}
	\end{figure}
	
	The result shows that with unbalanced data and our adaptive choice of $T_i$, the advantage of our method becomes more obvious, especially for MNIST data. The comparison of Figure \ref{fig:unbalance}(b) with Figure  \ref{fig:mnist}(d) shows that our method achieves nearly the same performance as the balanced case, while other methods are somewhat negatively affected by the unbalanced data allocation.
	

\subsection{Heterogeneous Data}
Finally, we show the robustness of our proposed method for non-i.i.d data. In particular, we make clients to be heterogeneous. We make some adjustment to the original linear model \eqref{eq:model}, such that in different machines, $\mathbf{w}^*$ are slightly different:
\begin{eqnarray}
	V_j = \langle \mathbf{U}_j, \mathbf{w}^*+\mathbf{\Delta}_\mathbf{w}(i)\rangle+W_j
	\label{eq:modelnew}
\end{eqnarray}
for all sample $j$ stored in client $i$. $\Delta_\mathbf{w}(i)$ is the additional weight for client $i$. We let each component of $\Delta_\mathbf{w}(i)$ follows distribution $\mathcal{N}(0,0.2)$. Under this setting, conditional on $\Delta_\mathbf{w}(i)$, samples within each clients are i.i.d. However, they are no longer i.i.d between different clients. The results are shown in Figure \ref{fig:niid}, in which we fix $\epsilon=0.2$ here.

\begin{figure}[h!]
	\centering	
	\begin{subfigure}{0.33\linewidth}
		\includegraphics[width=\textwidth,height=0.8\textwidth]{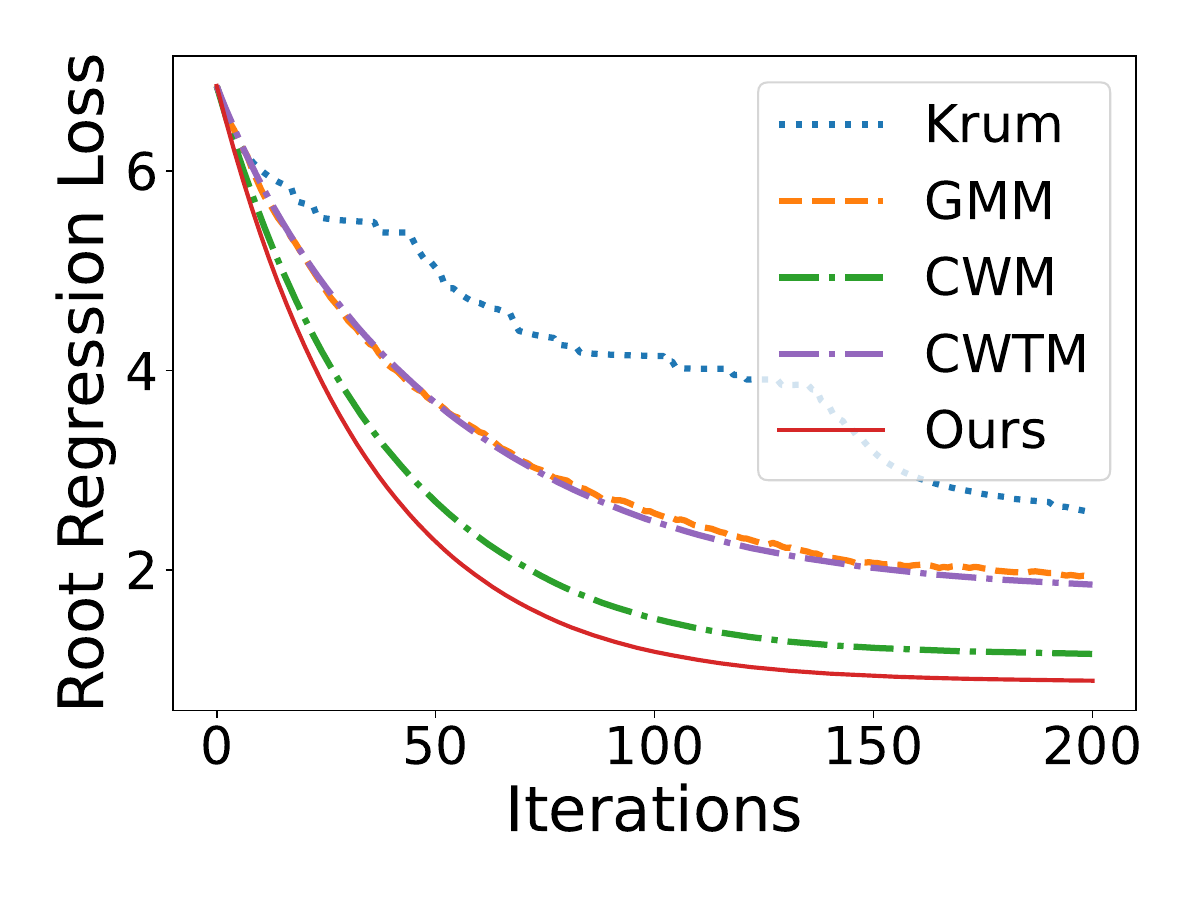}
		\caption{KA.}
	\end{subfigure}
	\begin{subfigure}{0.33\linewidth}
		\includegraphics[width=\textwidth,height=0.8\textwidth]{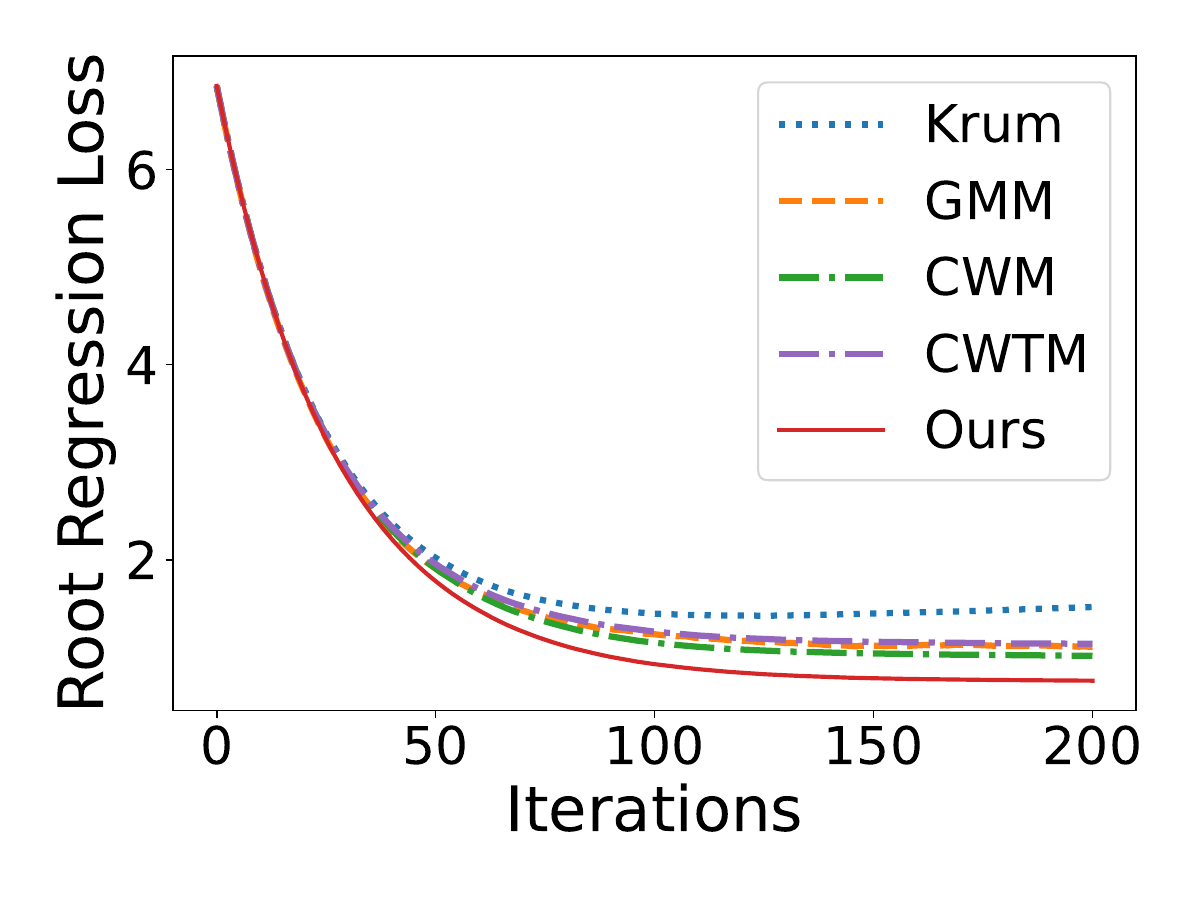}
		\caption{TMA.}
	\end{subfigure}
	\caption{Experiments on non-i.i.d case under model \eqref{eq:modelnew}, with $\epsilon=0.2$. }\label{fig:niid}
\end{figure}

According to Figure \ref{fig:niid}, with heterogeneous clients, our method still exhibits advantage compared to existing ones.
	\section{Conclusion}\label{sec:conc}
	In this paper, we have proposed a novel approach for Byzantine robust federated learning based on Huber loss minimization. Theoretical analyses have been conducted under the initial i.i.d assumption with balanced data, and subsequently extended to unbalanced and heterogeneous scenarios. Our method offers several advantages over existing approaches, including optimal statistical rate with fixed dimension, convenient selection of parameters without knowledge of Byzantine fraction $\epsilon$, and suitability for clients with unbalanced data. Furthermore, we have presented an algorithm to implement the multi-dimensional Huber loss minimization. The effectiveness of our approach is validated by numerical experiments. There are several possible directions for further improvements in the future. For example, pre-processing by filtering some obvious outliers may help to reduce the effect of Byzantine clients. Moreover, it is also worthwhile to consider one-round or multi-round robust algorithms to reduce the communication cost.

\bibliographystyle{ieeetr}
\bibliography{distributed}

\appendix

\section{Common Lemmas on Convergence}

Recall that the aggregator function is defined as 
\begin{eqnarray}
	g(\mathbf{w}_t) = \underset{s}{\arg\min}\sum_{i=1}^m n_i\phi_i(\norm{\mathbf{s}-\mathbf{X}_{it}}).
\end{eqnarray}
Now we generalize the definition above to arbitrary $\mathbf{w}\in \mathcal{W}$:
\begin{eqnarray}
	g(\mathbf{w}) = \underset{s}{\arg\min}\sum_{i=1}^m n_i\phi_i(\norm{\mathbf{s}-\mathbf{X}_i(\mathbf{w})}),
	\label{eq:agg}
\end{eqnarray}
in which $\mathbf{X}_i(\mathbf{w})$ denotes the signal from client $i$, if the master broadcasts $\mathbf{w}$. In particular,
\begin{eqnarray}
	\mathbf{X}_i(\mathbf{w})=\left\{
	\begin{array}{ccc}
		\mathbf{G}_i(\mathbf{w}) &\text{if} & i\notin \mathcal{B}\\
		\star &\text{if} & i\in \mathcal{B},
	\end{array}
	\right.
	\label{eq:xi}
\end{eqnarray}
with
\begin{eqnarray}
	\mathbf{G}_i(\mathbf{w}) = \frac{1}{n_i}\sum_{j=1}^{n_i}\nabla f(\mathbf{w}, \mathbf{Z}_{ij}).
	\label{eq:gi}
\end{eqnarray}
Define
\begin{eqnarray}
	\Delta = \underset{\pi_A}{\sup}\underset{\mathbf{w}\in \mathcal{W}}{\sup}\norm{g(\mathbf{w}) - \nabla F(\mathbf{w})}
	\label{eq:delta}
\end{eqnarray}
as the supremum error of gradient aggregator, in which $\pi_A$ stands for the strategy of the adversary.

Then we show the following lemmas.

\begin{lem}\label{lem:strong}
	(Strong convex $F$) If $F$ is $\mu$-strong convex and $L$-smooth, $\mathcal{W}$ is a convex set, and $\mathbf{w}^*\in \mathcal{W}$, then with $\eta\leq 1/L$,
	\begin{eqnarray}
		\norm{\mathbf{w}_t-\mathbf{w}^*}\leq (1-\rho)^t \norm{\mathbf{w}_0-\mathbf{w}^*}+\frac{2\Delta}{\mu},
	\end{eqnarray}
	with 
	\begin{eqnarray}
		\rho =\frac{1}{2}\eta\mu.
	\end{eqnarray}
\end{lem}
\begin{proof}
	The proof is shown in Section \ref{sec:strong}.
\end{proof}
\begin{lem}\label{lem:convex}
	(General convex $F$) If $F$ is convex and $L$-smooth, and
	\begin{eqnarray}
		\mathcal{W}=\{\mathbf{w}|\norm{\mathbf{w}-\mathbf{w}^*}\leq 2\norm{\mathbf{w}_0-\mathbf{w}^*} \},	
		\label{eq:space}
	\end{eqnarray}  then with $\eta = 1/L$, after $t_m=(L/\Delta)\norm{\mathbf{w}_0-\mathbf{w}^*}_2$ steps,
	\begin{eqnarray}
		F(\mathbf{w}_{t_m})-F(\mathbf{w}^*)\leq 16\norm{\mathbf{w}_0-\mathbf{w}^*}\Delta.
	\end{eqnarray}
\end{lem}
\begin{proof}
	The proof is shown in Section \ref{sec:convex}.
\end{proof}
\begin{lem}\label{lem:nonconvex}
	(Non-convex $F$) If $F$ is $L$-smooth, $\norm{\nabla F(\mathbf{w})}\leq M$ for all $\mathbf{w}\in \mathcal{W}$, and
	\begin{eqnarray}
		\mathcal{W}=\left\{\mathbf{w}|\norm{\mathbf{w}-\mathbf{w}^*}\leq \frac{2}{\Delta^2}(M+\Delta)(F(\mathbf{w}_0)-F(\mathbf{w}^*)) \right\},
	\end{eqnarray}
	then with $\eta=1/L$ and
	\begin{eqnarray}
		t_m=\frac{2L}{\Delta^2}(F(\mathbf{w}_0) - F(\mathbf{w}^*)),
		\label{eq:Tnonconvex}
	\end{eqnarray}
	we have
	\begin{eqnarray}
		\underset{t=0,1,\ldots, t_m}{\min}\norm{\nabla F(\mathbf{w}_t)}\leq \sqrt{2}\Delta.
	\end{eqnarray}
\end{lem}
\begin{proof}
	The proof is shown in Section \ref{sec:nonconv}.
\end{proof}

\subsection{Proof of Lemma \ref{lem:strong}}\label{sec:strong}

\begin{eqnarray}
	\norm{\mathbf{w}_{t+1}-\mathbf{w}^*} &=& \norm{\Pi_{\mathcal{W}}(\mathbf{w}_t-\eta g(\mathbf{w}_t))-\mathbf{w}^*}\nonumber\\
	&\overset{(a)}{\leq} & \norm{\mathbf{w}_t-\eta g(\mathbf{w}_t) - \mathbf{w}^*}\nonumber\\
	&\leq & \norm{\mathbf{w}_t-\eta\nabla F(\mathbf{w}_t) - \mathbf{w}^*}+\eta \norm{\nabla F(\mathbf{w}_t) - g(\mathbf{w}_t)}\nonumber\\
	&\overset{(b)}{\leq} & \norm{\mathbf{w}_t-\eta\nabla F(\mathbf{w}_t)-\mathbf{w}^*}+\eta \Delta,
	\label{eq:iteration}
\end{eqnarray}
in which (a) uses the assumption that $\mathcal{W}$ is a convex set, which was made in the statement of Lemma \ref{lem:strong}. (b) uses \eqref{eq:delta}. For the first term,
\begin{eqnarray}
	&&\norm{\mathbf{w}_t-\eta\nabla F(\mathbf{w}_t) - \mathbf{w}^*}^2\nonumber\\
	&=& \norm{\mathbf{w}_t-\mathbf{w}^*}^2 - 2\eta \langle \mathbf{w}_t-\mathbf{w}^*, \nabla F(\mathbf{w}_t)\rangle + \eta^2 \norm{\nabla F(\mathbf{w}_t)}^2\nonumber\\
	&\overset{(a)}{\leq} & \norm{\mathbf{w}_t-\mathbf{w}^*}^2 - 2\eta \left(F(\mathbf{w}_t) - F(\mathbf{w}^*)+\frac{\mu}{2}\norm{\mathbf{w}_t-\mathbf{w}^*}^2\right) + \eta^2 \norm{\nabla F(\mathbf{w}_t)}^2\nonumber\\
	&\overset{(b)}{\leq}& (1-\eta\mu)\norm{\mathbf{w}_t-\mathbf{w}^*}^2 - 2\eta (F(\mathbf{w}_t) - F(\mathbf{w}^*)) + 2\eta^2 L (F(\mathbf{w}_t) - F(\mathbf{w}^*))\nonumber\\
	&\overset{(c)}{\leq} & (1-\eta \mu)\norm{\mathbf{w}_t-\mathbf{w}^*}^2.
	\label{eq:iterfirst}
\end{eqnarray}
(a) comes from the $\mu$-strong convexity of $F$, which ensures that
\begin{eqnarray}
	F(\mathbf{w}^*)\geq F(\mathbf{w}_t) + \langle \nabla F(\mathbf{w}_t), \mathbf{w}^*-\mathbf{w}_t\rangle + \frac{\mu}{2}\norm{\mathbf{w}_t-\mathbf{w}^*}^2.
\end{eqnarray}
(b) comes from the $L$-smoothness of $F$, which ensures that
\begin{eqnarray}
	\norm{\nabla F(\mathbf{w}_t)}^2\leq 2L(F(\mathbf{w}_t) - F(\mathbf{w}^*)).
\end{eqnarray}
(c) is from the requirement $\eta \leq 1/L$. Therefore, from \eqref{eq:iteration} and \eqref{eq:iterfirst},
\begin{eqnarray}
	\norm{\mathbf{w}_{t+1}-\mathbf{w}^*}\leq \sqrt{1-\eta\mu}\norm{\mathbf{w}_t-\mathbf{w}^*}+\eta\Delta.
\end{eqnarray}
By induction, it can be shown that
\begin{eqnarray}
	\norm{\mathbf{w}_t-\mathbf{w}^*}\leq (1-\rho)^t\norm{\mathbf{w}_0-\mathbf{w}^*}+\frac{2\Delta}{\mu},
	\label{eq:final}
\end{eqnarray}
with $\rho=\eta\mu/2$.

\subsection{Proof of Lemma \ref{lem:convex}}\label{sec:convex}
Our proof follows \cite{yin2018byzantine} and \cite{zhu2023byzantine} with some slight improvements.

To begin with, we prove that $\mathbf{w}_t-\eta g(\mathbf{w}_t)\in \mathcal{W}$ for $t=1,\ldots, t_m-1$, such that the projection has no effect, i.e. $\Pi_\mathcal{W}(\mathbf{w}_t-\eta g(\mathbf{w}_t)) = \mathbf{w}_t-\eta g(\mathbf{w}_t)$. Note that
\begin{eqnarray}
	\norm{\mathbf{w}_t-\eta g(\mathbf{w}_t) - \mathbf{w}^*}\leq \norm{\mathbf{w}_t-\eta\nabla F(\mathbf{w}_t)-\mathbf{w}^*}+\eta\Delta.
\end{eqnarray}
Recall that $\eta=1/L$,
\begin{eqnarray}
	\norm{\mathbf{w}_t-\eta\nabla F(\mathbf{w}_t)-\mathbf{w}^*}^2 &=&\norm{\mathbf{w}_t-\mathbf{w}^*}^2 - 2\eta\langle \nabla F(\mathbf{w}_t), \mathbf{w}_t-\mathbf{w}^*\rangle + \eta^2 \norm{\nabla F(\mathbf{w}_t)}^2\nonumber\\
	&\leq & \norm{\mathbf{w}_t-\mathbf{w}^*}^2 - \frac{2\eta}{L}\norm{\nabla F(\mathbf{w}_t)}^2+\eta^2 \norm{\nabla F(\mathbf{w}_t)}^2\nonumber\\
	&=&\norm{\mathbf{w}_t-\mathbf{w}^*}^2 -\frac{1}{L^2}\norm{\nabla F(\mathbf{w}_t)}^2\nonumber\\
	&\leq & \norm{\mathbf{w}_t-\mathbf{w}^*}^2.
\end{eqnarray}
Therefore
\begin{eqnarray}
	\norm{\mathbf{w}_t-\eta g(\mathbf{w}_t)-\mathbf{w}^*}\leq \norm{\mathbf{w}_t-\mathbf{w}^*}+\frac{\Delta}{L}.
\end{eqnarray}
Now prove that for $t=0,1,\ldots, t_m-1$,
\begin{eqnarray}
	\norm{\mathbf{w}_t-g(\mathbf{w}_t)-\mathbf{w}^*}\leq \norm{\mathbf{w}_0-\mathbf{w}^*}+\frac{t\Delta}{L}
	\label{eq:iter}
\end{eqnarray}
by induction. \eqref{eq:iter} for $t=0$ is trivial. If \eqref{eq:iter} holds for $t$, then
\begin{eqnarray}
	\norm{\mathbf{w}_t-\eta g(\mathbf{w}_t)-\mathbf{w}^*}\leq\norm{\mathbf{w}_0-\mathbf{w}^*} +\frac{t_m\Delta}{L}\leq 2\norm{\mathbf{w}_0-\mathbf{w}^*},
\end{eqnarray}
hence $\mathbf{w}_t-g(\mathbf{w}_t)\in \mathcal{W}$, and $\mathbf{w}_{t+1}=\Pi_\mathcal{W}(\mathbf{w}_t-\eta g(\mathbf{w}_t))=\mathbf{w}_t-\eta g(\mathbf{w}_t)$. Hence, at time step $t+1$,
\begin{eqnarray}
	\norm{\mathbf{w}_{t+1} - \eta g(\mathbf{w}_{t+1})-\mathbf{w}^*}\leq \norm{\mathbf{w}_{t+1}-\mathbf{w}^*}+\frac{\Delta}{L}\leq \norm{\mathbf{w}_0-\mathbf{w}^*}+\frac{(t+1)\Delta}{L}.
\end{eqnarray}
Proof of \eqref{eq:iter} is complete. With \eqref{eq:iter}, it can be shown that $\mathbf{w}_t-\eta g(\mathbf{w}_t)\in \mathcal{W}$ for all $t=0,1,\ldots, t_m-1$. We can then study the algorithm without projection.

Since $F$ is $L$-smooth,
\begin{eqnarray}
	F(\mathbf{w}_{t+1})&\leq& F(\mathbf{w}_t) + \langle \nabla F(\mathbf{w}_t), \mathbf{w}_{t+1}-\mathbf{w}_t\rangle+\frac{L}{2}\norm{\mathbf{w}_{t+1}-\mathbf{w}_t}^2\nonumber\\
	&=& F(\mathbf{w}_t)-\eta \langle \nabla F(\mathbf{w}_t), g(\mathbf{w}_t)\rangle + \frac{1}{2}L\eta^2 \norm{g(\mathbf{w}_t)}^2\nonumber\\
	&=& F(\mathbf{w}_t)-\eta\langle \nabla F(\mathbf{w}_t), g(\mathbf{w}_t) - \nabla F(\mathbf{w}_t)\rangle - \eta\norm{\nabla F(\mathbf{w}_t)}^2\nonumber\\
	&&+\frac{1}{2}L\eta^2\left(\norm{\nabla F(\mathbf{w}_t)}^2 + 2\langle \nabla F(\mathbf{w}_t), g(\mathbf{w}_t)-\nabla F(\mathbf{w}_t)\rangle + \norm{g(\mathbf{w}_t) - \nabla F(\mathbf{w}_t)}^2\right)\nonumber\\
	&=& F(\mathbf{w}_t)-\frac{1}{2L}\norm{\nabla F(\mathbf{w}_t)}^2 + \frac{1}{2L}\Delta^2.
	\label{eq:trans}
\end{eqnarray}
We discuss two cases.

\textbf{Case 1.} $\norm{\nabla F(\mathbf{w}_t)}\geq \sqrt{2}\Delta$ for all $t=0,1,\ldots, t_m-1$. Then from \eqref{eq:trans}, define $D=\norm{\mathbf{w}_0-\mathbf{w}^*}$, then
\begin{eqnarray}
	F(\mathbf{w}_{t+1}) - F(\mathbf{w}^*)&\leq & F(\mathbf{w}_t) - F(\mathbf{w}^*)-\frac{1}{4L}\norm{\nabla F(\mathbf{w}_t)}^2\nonumber\\
	&\leq & F(\mathbf{w}_t)-F(\mathbf{w}^*)-\frac{1}{4L\norm{\mathbf{w}-\mathbf{w}^*}^2} (F(\mathbf{w}_t) - F(\mathbf{w}^*))^2\nonumber\\
	&\leq & F(\mathbf{w}_t)-F(\mathbf{w}^*)-\frac{1}{16LD^2}(F(\mathbf{w}_t) - F(\mathbf{w}^*))^2,
\end{eqnarray}
in which the last step comes from \eqref{eq:space}. Denote 
\begin{eqnarray}
	e_t:=F(\mathbf{w}_t)-F(\mathbf{w}^*),
\end{eqnarray}
then
\begin{eqnarray}
	e_{t+1}\leq e_t-\frac{1}{16LD^2}e_t^2,
\end{eqnarray}
and
\begin{eqnarray}
	\frac{1}{e_{t+1}} = \frac{1}{e_t\left(1-\frac{1}{16LD^2}e_t\right)}\geq \frac{1}{e_t}\left(1+\frac{1}{16LD^2}e_t\right)=\frac{1}{e_t} + \frac{1}{16LD^2},
\end{eqnarray}
hence
\begin{eqnarray}
	\frac{1}{e_{t_m}}\geq \frac{t_m}{16LD^2},
\end{eqnarray}
thus
\begin{eqnarray}
	F(\mathbf{w}_{t_m})-F(\mathbf{w}^*)\leq \frac{16LD^2}{t_m}=16D\Delta.
	\label{eq:case1}
\end{eqnarray}
\textbf{Case 2.} $\norm{\nabla F(\mathbf{w}_k)}<\sqrt{2}\Delta$ for some $k\in \{0,1,\ldots, t_m-1\}$. Then for all $k\in \{0,1,\ldots, t_m-1\}$, from the convexity of $F$, 
\begin{eqnarray}
	F(\mathbf{w}_k)-F(\mathbf{w}^*)\leq \norm{\nabla F(\mathbf{w}_k)}\norm{\mathbf{w}_k-\mathbf{w}^*}\leq 2\sqrt{2}D\Delta.
	\label{eq:errk}
\end{eqnarray}
Now prove
\begin{eqnarray}
	F(\mathbf{w}_t) - F(\mathbf{w}^*)\leq 2\sqrt{2}D\Delta,
	\label{eq:errk2}
\end{eqnarray}
for $t=k,k+1,\ldots, t_m$ by induction. \eqref{eq:errk2} with $t=k$ obviously hold. If \eqref{eq:errk2} holds for $t$, at $t+1$, if $\norm{\nabla F(\mathbf{w}_{t+1})}<\sqrt{2}\Delta$, then \eqref{eq:errk2} holds. Otherwise, by \eqref{eq:trans},
\begin{eqnarray}
	F(\mathbf{w}_{t+1})-F(\mathbf{w}^*)\leq F(\mathbf{w}_t)-F(\mathbf{w}^*)- \frac{1}{4L}\norm{\nabla F(\mathbf{w}_t)}^2\leq F(\mathbf{w}_t)-F(\mathbf{w}^*)\leq 2\sqrt{2}D\Delta.
\end{eqnarray}
Therefore \eqref{eq:errk2} is proved. Hence
\begin{eqnarray}
	F(\mathbf{w}_{t_m}) - F(\mathbf{w}^*)\leq 2\sqrt{2}D\Delta.
	\label{eq:case2}	
\end{eqnarray}
Combine \eqref{eq:case1} for Case 1 and \eqref{eq:case2} for Case 2, and recall that $D=\norm{\mathbf{w}_0-\mathbf{w}^*}$,
\begin{eqnarray}
	F(\mathbf{w}_{t_m}) - F(\mathbf{w}^*)\leq 16\norm{\mathbf{w}_0-\mathbf{w}^*}\Delta.	
\end{eqnarray}
The proof of Lemma \ref{lem:convex} is complete.

\subsection{Proof of Lemma \ref{lem:nonconvex}}\label{sec:nonconv}
\begin{eqnarray}
	\norm{\mathbf{w}_t-\eta g(\mathbf{w}_t) - \mathbf{w}^*}&\leq& \norm{\mathbf{w}_t-\mathbf{w}^*}+\eta (\norm{\nabla F(\mathbf{w})} +\norm{g(\mathbf{w}) - \nabla F(\mathbf{w})})\nonumber\\
	&\leq & \norm{\mathbf{w}_t-\mathbf{w}^*}+\frac{1}{L}(M+\Delta).
\end{eqnarray}
Run $T$ steps with $T$ determined in \eqref{eq:Tnonconvex}, 
\begin{eqnarray}
	\norm{\mathbf{w}_t-\eta g(\mathbf{w}_t) - \mathbf{w}^*}\leq \frac{2}{\Delta^2}(M+\Delta)(F(\mathbf{w}_0) - F(\mathbf{w}^*)).
\end{eqnarray}
Hence we can study the algorithm without projection. From \eqref{eq:trans},
\begin{eqnarray}
	F(\mathbf{w}_{t+1})\leq  F(\mathbf{w}_0) - \frac{1}{2L}\norm{\nabla F(\mathbf{w}_t)}+\frac{1}{2L}\Delta^2.
\end{eqnarray}
Sum up over $t$,
\begin{eqnarray}
	F(\mathbf{w}_{t_m})\leq F(\mathbf{w}_0) - \frac{1}{2L}\sum_{t=0}^{t_m-1}\norm{\nabla F(\mathbf{w}_t)}^2+\frac{1}{2L}\Delta^2.
\end{eqnarray}
Hence
\begin{eqnarray}
	\sum_{t=0}^{t_m-1}\norm{\nabla F(\mathbf{w}_t)}^2\leq 2L(F(\mathbf{w}_0) - F(\mathbf{w}^*))+\Delta^2\leq 2\Delta^2,
\end{eqnarray}
and
\begin{eqnarray}
	\underset{t=0,1,\ldots, t_m}{\min}\norm{\nabla F(\mathbf{w}_t)}\leq \sqrt{2}\Delta.
\end{eqnarray}
The proof of Lemma \ref{lem:nonconvex} is complete.

\section{Proof of Theorem 1}\label{sec:simple}
Define
\begin{eqnarray}
	r_0 = \max\left\{\sqrt{\frac{8\sigma^2}{n}\ln \frac{2\times 6^dmC_W(NL)^d}{\delta}}, \frac{4\sigma}{n}\ln \frac{2\times 6^dmC_W(NL)^d}{\delta} \right\}+\frac{2}{N},
	\label{eq:r0}
\end{eqnarray}
and
\begin{eqnarray}
	\Delta_0=\max\left\{\sqrt{\frac{8\sigma^2}{N}\ln \frac{2\times 6^dC_W(NL)^d}{\delta}}, \frac{4\sigma}{N}\ln \frac{2\times 6^dC_W(NL)^d}{\delta} \right\}+\frac{2}{N}.
	\label{eq:delta0}
\end{eqnarray}
We provide a finite sample statement of Theorem \ref{thm:simple} as following. 
\begin{theorem}\label{thm:simple_app}
	If $T\geq 3r_0$, then under Assumption 2 and 3, the following equations hold with probability at least $1-\delta$.
	
	(1) (Strong convex $F$) Under Assumption 1(a), with $\eta\leq 1/L$,
	\begin{eqnarray}
		\norm{\mathbf{w}_t-\mathbf{w}^*}\leq (1-\rho)^t \norm{\mathbf{w}_0-\mathbf{w}^*}+\frac{2\Delta_A}{\mu},
	\end{eqnarray}
	in which $\rho = \eta\mu/2$;
	
	(2) (General convex $F$) Under Assumption 1(b), with $\eta=1/L$, after $t_m=(L/\Delta_A)\norm{\mathbf{w}_0-\mathbf{w}^*}_2$ steps,
	\begin{eqnarray}
		F(\mathbf{w}_{t_m})-F(\mathbf{w}^*)\leq 16\norm{\mathbf{w}_0-\mathbf{w}^*}\Delta_A;
	\end{eqnarray}
	
	(3) (Non-convex $F$) Under Assumption 1(c), with $\eta = 1/L$, after
	$t_m=(2L/\Delta_A^2)(F(\mathbf{w}_0) - F(\mathbf{w}^*))$ steps, we have
	\begin{eqnarray}
		\underset{t=0,1,\ldots, t_m}{\min}\norm{\nabla F(\mathbf{w}_t)}\leq \sqrt{2}\Delta_A,
	\end{eqnarray}
	in which 
	\begin{eqnarray}
		\Delta_A=\left\{
		\begin{array}{ccc}
			2\epsilon T+\Delta_0&\text{if} & \epsilon\leq \frac{1}{4}\nonumber\\
			\frac{1-\epsilon}{\sqrt{1-2\epsilon}}r_0+T &\text{if} & \frac{1}{4}<\epsilon<\frac{1}{2}.	
		\end{array}
		\right.
	\end{eqnarray}
\end{theorem}
We would like to remark that with $T\sim \sigma\sqrt{(d/n)\ln (N/\delta)}$, we have the following asymptotic bound
\begin{eqnarray}
	\Delta_A\lesssim \left(\frac{1}{\sqrt{1-2\epsilon}}\frac{\epsilon}{\sqrt{n}}+\frac{1}{\sqrt{N}}\right)\sqrt{d\ln \frac{N}{\delta}}.
\end{eqnarray}

Now we prove this theorem. Recall that $\mathbf{X}_i(\mathbf{w})$ and $\mathbf{G}_i(\mathbf{w})$ are defined in \eqref{eq:xi} and \eqref{eq:gi}.

\begin{lem}\label{lem:hpb}
	With probability at least $1-\delta$, the following inequalities hold: (1) For all $i\in [m]$ and $\mathbf{w}\in \mathcal{W}$,
	\begin{eqnarray}
		\norm{\mathbf{G}_i(\mathbf{w})-\nabla F(\mathbf{w})}\leq r_0;
		\label{eq:b1}
	\end{eqnarray}
	
	(2) For all $\mathbf{w}\in \mathcal{W}$,
	\begin{eqnarray}
		\norm{\frac{1}{m}\sum_{i=1}^m \mathbf{G}_i(\mathbf{w}) - \nabla F(\mathbf{w})}\leq \Delta_0.
		\label{eq:b2}
	\end{eqnarray}
\end{lem}
\begin{proof}
	The proof is shown in Section \ref{sec:hpb}.
\end{proof}
Define
\begin{eqnarray}
	a(\mathbf{w})=\underset{s}{\arg\min}\sum_{i\in [m]\setminus \mathcal{B}} \phi(\norm{\mathbf{s}-\mathbf{G}_i(\mathbf{w})}).
	\label{eq:adef}
\end{eqnarray}
Intuitively, $a(\mathbf{w})$ is the aggregated gradient if all Byzantine nodes are removed. Then the estimation error of gradient $\norm{g(\mathbf{w})-\nabla F(\mathbf{w})}$ can be bounded by $\norm{a(\mathbf{w})-\nabla F(\mathbf{w})}+\norm{g(\mathbf{w})-a(\mathbf{w})}$. Following this idea, we show some lemmas.
\begin{lem}\label{lem:dist}
	If \eqref{eq:b1} holds uniformly for all $i$ and $\mathbf{w}\in \mathcal{W}$, then for all $i=1,\ldots, m$,
	\begin{eqnarray}
		\norm{a(\mathbf{w}) - \mathbf{G}_i(\mathbf{w})}\leq 2r_0.
	\end{eqnarray}
\end{lem}
\begin{proof}
	The proof is shown in Section \ref{sec:dist}.
\end{proof}
\begin{lem}\label{lem:diff1}
	Under the condition that \eqref{eq:b1} and \eqref{eq:b2} hold uniformly for all $i$ and $\mathbf{w}\in \mathcal{W}$, if 
	\begin{eqnarray}
		T\geq \frac{2(1-\epsilon)}{1-2\epsilon} r_0,
		\label{eq:Tlb}
	\end{eqnarray}
	then for all $\mathbf{w}\in \mathcal{W}$,
	\begin{eqnarray}
		\norm{g(\mathbf{w}) - a(\mathbf{w})}\leq \frac{\epsilon T}{1-\epsilon},
	\end{eqnarray}
	in which $r_0$ is defined in \eqref{eq:r0}.
\end{lem}
\begin{proof}
	The proof is shown in Section \ref{sec:diff1}.
\end{proof}
\begin{lem}\label{lem:diff2}
	Under the same conditions as Lemma \ref{lem:diff1}, for all $\mathbf{w}\in \mathcal{W}$,
	\begin{eqnarray}
		\norm{a(\mathbf{w}) - \nabla F(\mathbf{w})}\leq \Delta_0,
	\end{eqnarray}
	in which $\Delta_0$ is defined in \eqref{eq:delta0}.
\end{lem}
\begin{proof}
	The proof is shown in Section \ref{sec:diff2}.
\end{proof}
\begin{lem}\label{lem:diffnew}
	Under the condition that \eqref{eq:b1} and \eqref{eq:b2} hold uniformly for all $i$ and $\mathbf{w}\in \mathcal{W}$, 
	\begin{eqnarray}
		\norm{g(\mathbf{w}) - \nabla F(\mathbf{w})}\leq \max\left\{\frac{1-\epsilon}{\sqrt{1-2\epsilon}}r_0, T+r_0 \right\}.
	\end{eqnarray}
\end{lem}
\begin{proof}
	The proof is shown in Section \ref{sec:diffnew}.
\end{proof}

Now we combine the results of Lemma \ref{lem:diff1}, \ref{lem:diff2} and \ref{lem:diffnew}. With $T\geq 3r_0$, if $\epsilon\leq 1/4$, then \eqref{eq:Tlb} holds. From Lemma \ref{lem:diff1} and \ref{lem:diff2},
\begin{eqnarray}
	\norm{g(\mathbf{w})-\nabla F(\mathbf{w})}\leq \frac{\epsilon T}{1-\epsilon}+\Delta_0\leq 2\epsilon T+\Delta_0.
\end{eqnarray}
If $\epsilon>1/4$, then from Lemma \ref{lem:diffnew},
\begin{eqnarray}
	\norm{g(\mathbf{w})-\nabla F(\mathbf{w})}\leq \frac{1-\epsilon}{\sqrt{1-2\epsilon}}r_0+T.
\end{eqnarray}
Define
\begin{eqnarray}
	\Delta_A=\left\{
	\begin{array}{ccc}
		2\epsilon T+\Delta_0&\text{if} & \epsilon\leq \frac{1}{4}\nonumber\\
		\frac{1-\epsilon}{\sqrt{1-2\epsilon}}r_0+T &\text{if} & \frac{1}{4}<\epsilon<\frac{1}{2}.	
	\end{array}
	\right.
\end{eqnarray}
Then
\begin{eqnarray}
	\norm{g(\mathbf{w}) - \nabla F(\mathbf{w})} \leq \Delta_A.
	\label{eq:deltaa_app}
\end{eqnarray}
Theorem \ref{thm:simple_app} can then be proved using Lemma \ref{lem:strong}, \ref{lem:convex} and \ref{lem:nonconvex} with $\Delta=\Delta_A$.
\subsection{Proof of Lemma \ref{lem:hpb}}\label{sec:hpb}
Here we show that both (1) and (2) in Lemma \ref{lem:hpb} are satisfied with probability at least $1-\delta/2$.

Define $l=1/(NL)$. With $N> 1/(r_DL),l<r_D$ holds, in which $r_D$ is the constant in Assumption 2. Let $\mathbf{w}^k$, $k=1,\ldots, N_c(l)$ be a $l$-covering of the parameter space $\mathcal{W}$. Then according to Assumption 2, the covering number is bounded by
\begin{eqnarray}
	N_c(l)\leq \frac{C_W}{l^d} = C_W(NL)^d.
\end{eqnarray}

Firstly, we prove the sub-exponential concentration of $\mathbf{G}_i$. For all vector $\mathbf{v}$ with $\norm{\mathbf{v}} = 1$, we have
\begin{eqnarray}
	\mathbb{E}\left[e^{\lambda \mathbf{v}^T (\mathbf{G}_i(\mathbf{w}) - \nabla F(\mathbf{w}))}\right] &=& \mathbb{E}\left[\exp\left[\frac{1}{n}\lambda \mathbf{v}^T \sum_{j=1}^n (\nabla f(\mathbf{w}, \mathbf{Z}_{ij}) - \nabla F(\mathbf{w}))\right]\right]\nonumber\\
	&\leq & \left(e^{\frac{1}{2}\sigma^2\frac{\lambda^2}{n^2}}\right)^n \text{ if }  |\lambda|\leq \frac{n}{\sigma}\nonumber\\
	&=& e^{\frac{1}{2n} \sigma^2 \lambda^2} \text{ if } |\lambda|\leq \frac{n}{\sigma}.
	\label{eq:Gsubexp}	
\end{eqnarray}
Hence
\begin{eqnarray}
	\text{P}(\mathbf{v}^T (G_i(\mathbf{w}) - \nabla F(\mathbf{w})) > u)&\leq & \underset{\lambda \geq 0}{\inf} e^{-\lambda u}\mathbb{E}\left[e^{\lambda \mathbf{v}^T (G_i(\mathbf{w}) - \nabla F(\mathbf{w}))}\right]\nonumber\\
	&\leq & \underset{0\leq \lambda\leq n/\sigma}{\inf} e^{-\lambda u} e^{\frac{1}{2n}\sigma^2 \lambda^2}\nonumber\\
	&\leq & e^{-\frac{n}{2}\min\left\{\frac{u^2}{\sigma^2}, \frac{u}{\sigma} \right\}}.
\end{eqnarray}
Let $\mathcal{V}=\{v_1,\ldots, v_{N_0} \}$ be a $1/2$-covering of unit sphere. From Lemma 5.2 and 5.3 in \cite{vershynin2010introduction}, we have $N_0\leq 6^d$, and
\begin{eqnarray}
	\norm{\mathbf{G}_i(\mathbf{w}) - \nabla F(\mathbf{w})}\leq 2\underset{\mathbf{v}\in \mathcal{V}}{\sup} \mathbf{v}^T (\mathbf{G}_i(\mathbf{w}) - \nabla F(\mathbf{w})).
\end{eqnarray}
Such argument was also used in \cite{chen2017distributed}. Hence
\begin{eqnarray}
	\text{P}\left(\norm{G_i(\mathbf{w}) - \nabla F(\mathbf{w})}>u\right) &\leq & \text{P}\left(\cup_{\mathbf{v}\in \mathcal{V}}\left\{\mathbf{v}^T(\mathbf{G}_i(\mathbf{w}) - \nabla F(\mathbf{w}))>\frac{u}{2} \right\}\right)\nonumber\\
	&\leq & 6^de^{-\frac{n}{2}\min\left\{\frac{u^2}{4\sigma^2}, \frac{u}{2\sigma} \right\}}.
	\label{eq:largeprob}
\end{eqnarray}
Take union, we get
\begin{eqnarray}
	\text{P}\left(\exists i\in [m], \exists k \in [N_c(l)], \norm{\mathbf{G}_i(\mathbf{w}^k)-\nabla F(\mathbf{w}^k)}>u\right)\leq mC_W(NL)^d6^de^{-\frac{n}{2}\min\left\{\frac{u^2}{4\sigma^2}, \frac{u}{2\sigma} \right\}}.
\end{eqnarray}
Recall the definition of $r_0$ in \eqref{eq:r0}. Then
\begin{eqnarray}
	\text{P}\left(\exists i\in [m], \exists k \in [N_c(l)], \norm{\mathbf{G}_i(\mathbf{w}^k)-\nabla F(\mathbf{w}^k)}>r_0-\frac{2}{N}\right)\leq \frac{\delta}{2}.
\end{eqnarray}
To generalize above bound from $\mathbf{w}^k$ to arbitrary $\mathbf{w}\in \mathcal{W}$, note that 
\begin{eqnarray}
	&&\norm{\mathbf{G}_i(\mathbf{w}) - \nabla F(\mathbf{w})}\nonumber\\
	&=&\underset{k}{\min}\left[\norm{\mathbf{G}_i(\mathbf{w})-\mathbf{G}_i(\mathbf{w}^k)}+\norm{\mathbf{G}_i(\mathbf{w}^k)-\nabla F(\mathbf{w}^k)}+\norm{\nabla F(\mathbf{w}^k)-\nabla F(\mathbf{w})}\right]\nonumber\\
	&\leq & \underset{k}{\min}\left[\norm{\mathbf{G}_i(\mathbf{w})-\mathbf{G}_i(\mathbf{w}^k)}+\norm{\nabla F(\mathbf{w}^k)-\nabla F(\mathbf{w})}\right]+\underset{k}{\max}\norm{\mathbf{G}_i(\mathbf{w}^k)-\nabla F(\mathbf{w}^k)}\nonumber\\
	&\leq &2Ll+\underset{k}{\max}\norm{\mathbf{G}_i(\mathbf{w}^k)-\nabla F(\mathbf{w}^k)}\nonumber\\
	&=&\frac{2}{N}+\underset{k}{\max}\norm{\mathbf{G}_i(\mathbf{w}^k)-\nabla F(\mathbf{w}^k)}.
	\label{eq:togrid}
\end{eqnarray}
Hence
\begin{eqnarray}
	\text{P}\left(\exists i\in [m], \exists \mathbf{w}\in \mathcal{W}, \norm{\mathbf{G}_i(\mathbf{w}^k)-\nabla F(\mathbf{w}^k)}>r_0\right)\leq \frac{\delta}{2}.
	\label{eq:prob1}
\end{eqnarray}
Now we bound the probability of violating (2) in Lemma \ref{lem:hpb}. Similar to \eqref{eq:largeprob}, we have
\begin{eqnarray}
	\text{P}\left(\norm{\frac{1}{m}\sum_{i=1}^m \mathbf{G}_i(\mathbf{w})-\nabla F(\mathbf{w})}> u\right)\leq 6^de^{-\frac{N}{2}\min\left\{\frac{u^2}{4\sigma^2}, \frac{u}{2\sigma} \right\}}.
\end{eqnarray}
Recall the definition of $\Delta_0$ in \eqref{eq:delta0}. Then
\begin{eqnarray}
	\text{P}\left(\exists k\in [N_c(l)], \norm{\frac{1}{m}\sum_{i=1}^m \mathbf{G}_i(\mathbf{w}^k) - \nabla F(\mathbf{w}^k)}>\Delta_0-\frac{2}{N}\right)\leq \frac{\delta}{2}.
\end{eqnarray}
Therefore, with probability at least $1-\delta$, \eqref{eq:b1} and \eqref{eq:b2} are both satisfied for all $i\in [m]$ and $k\in [N_c(l)]$. The corresponding bound for all $\mathbf{w}\in \mathcal{W}$ is
\begin{eqnarray}
	\text{P}\left(\exists \mathbf{w}\in \mathcal{W}, \norm{\frac{1}{m}\sum_{i=1}^m \mathbf{G}_i(\mathbf{w}^k) - \nabla F(\mathbf{w}^k)}>\Delta_0\right)\leq \frac{\delta}{2}.
	\label{eq:prob2}
\end{eqnarray}
With \eqref{eq:prob1} and \eqref{eq:prob2}, the proof is finished.
\subsection{Proof of Lemma \ref{lem:dist}}\label{sec:dist}
From \eqref{eq:adef}, $a(\mathbf{w})$ is within the convex hull of $\mathbf{G}_i(\mathbf{w})$, $i=1,\ldots, m$. Therefore
\begin{eqnarray}
	\norm{a(\mathbf{w}) - \mathbf{G}_i(\mathbf{w})} &\leq & \underset{i'\in [m]}{\max}\norm{\mathbf{G}_{i'}(\mathbf{w})- \mathbf{G}_i(\mathbf{w})}\nonumber\\
	&\leq & \underset{i'\in [m]}{\max}\left( \norm{\mathbf{G}_{i'}(\mathbf{w}) - \nabla F(\mathbf{w})} + \norm{\mathbf{G}_i(\mathbf{w}) - \nabla F(\mathbf{w})}\right)\nonumber\\
	&\leq & 2r_0.
\end{eqnarray}

\subsection{Proof of Lemma \ref{lem:diff1}}\label{sec:diff1}
From \eqref{eq:adef}, we have
\begin{eqnarray}
	\sum_{i\in [m]\setminus \mathcal{B}} \nabla_s \phi(\norm{a(\mathbf{w}) - \mathbf{G}_i(\mathbf{w})}) = \mathbf{0}.
	\label{eq:astable}
\end{eqnarray}
Define
\begin{eqnarray}
	h(\mathbf{s}, \mathbf{w}) = \sum_{i=1}^m \phi(\norm{\mathbf{s}-\mathbf{X}_i(\mathbf{w})}).
	\label{eq:hdf}
\end{eqnarray}
To prove Lemma \ref{lem:diff1}, we calculate the gradient of $h$ at $g(\mathbf{w})$ and $a(\mathbf{w})$, respectively, and then bound $\norm{g(\mathbf{w})-a(\mathbf{w})}$ using the Hessian of $h$. Recall that the aggregator function $g(\mathbf{w})$ is defined in \eqref{eq:agg}, i.e.
\begin{eqnarray}
	g(\mathbf{w}) = \underset{s}{\arg\min}h(\mathbf{s}, \mathbf{w}),
\end{eqnarray}
hence
\begin{eqnarray}
	\nabla_sh(g(\mathbf{w}), \mathbf{w}) = \mathbf{0}.
	\label{eq:hgrad1}
\end{eqnarray}
Now we bound the gradient of $h$ at $a(\mathbf{w})$.
\begin{eqnarray}
	\norm{\nabla_s h(a(\mathbf{w}), \mathbf{w})}&=&\norm{\sum_{i=1}^m \nabla_s \phi(\norm{a(\mathbf{w}) - \mathbf{X}_i(\mathbf{w})})}\nonumber\\
	&=&\norm{\sum_{i\in [m]\setminus \mathcal{B}} \left(\nabla_s \phi(\norm{a(\mathbf{w}) - \mathbf{X}_i(\mathbf{w})}) \right)+\sum_{i\in  \mathcal{B}} \left(\nabla_s \phi(\norm{a(\mathbf{w}) - \mathbf{X}_i(\mathbf{w})}) \right)}\nonumber\\
	&\overset{(a)}{=}&\norm{\sum_{i\in  \mathcal{B}} \left(\nabla_s \phi(\norm{a(\mathbf{w}) - \mathbf{X}_i(\mathbf{w})}) \right)}\nonumber\\
	&\overset{(b)}{\leq}&|\mathcal{B}|T\nonumber\\
	&\leq & \epsilon mT.
	\label{eq:hgrad2}
\end{eqnarray}
In above steps, (a) comes from \eqref{eq:astable}. Note that $\mathbf{X}_i(\mathbf{w}) = \mathbf{G}_i(\mathbf{w})$ for $i\in [m]\setminus \mathcal{B}$. For (c), from \eqref{eq:phii}, the gradient of $\phi(\norm{\mathbf{s}-\mathbf{X}_i(\mathbf{w})})$ with respect to $\mathbf{s}$ is
\begin{eqnarray}
	\nabla_s \phi(\norm{s-\mathbf{X}_i(\mathbf{w})})=\left\{
	\begin{array}{ccc}
		\mathbf{s}-\mathbf{X}_i(\mathbf{w}) &\text{if} & \norm{\mathbf{s}-\mathbf{X}_i(\mathbf{w})}\leq T\\
		T\frac{\mathbf{s}-\mathbf{X}_i(\mathbf{w})}{\norm{\mathbf{s}-\mathbf{X}_i(\mathbf{w})}} &\text{if} & \norm{\mathbf{s}-\mathbf{X}_i(\mathbf{w})}> T.
	\end{array}
	\right.
	\label{eq:phigrad}
\end{eqnarray}
Therefore $\norm{\nabla_s \phi(\norm{s-\mathbf{X}_i(\mathbf{w})})}\leq T$ always holds. 

To bound $\norm{g(\mathbf{w}) - a(\mathbf{w})}$, it remains to bound the Hessian of $h$. By calculating derivative of \eqref{eq:phigrad}, $\nabla_s^2 \phi(\norm{\mathbf{s}-\mathbf{X}_i(\mathbf{w})}) = \mathbf{I}$ if $\norm{\mathbf{s}-\mathbf{X}_i(\mathbf{w})}\leq T$; if $\norm{\mathbf{s}-\mathbf{X}_i(\mathbf{w})}>T$, then $\nabla_s^2 \phi(\norm{\mathbf{s}-\mathbf{X}_i(\mathbf{w})})\succeq \mathbf{0}$. Define 
\begin{eqnarray}
	r_s=T-2r_0.
	\label{eq:rs}
\end{eqnarray}
Then for all $\mathbf{s}\in B(a(\mathbf{w}), r_s)$, we have
\begin{eqnarray}
	\norm{\mathbf{s}-\mathbf{G}_i(\mathbf{w})}&\leq& \norm{\mathbf{s}-\mathbf{a}(\mathbf{w})} + \norm{\mathbf{G}_i(\mathbf{w}) - a(\mathbf{w})}\nonumber\\
	&\leq &r_s + 2r_0\nonumber\\
	&=& T,
\end{eqnarray}
in which the second inequality comes from Lemma \ref{lem:dist}. Hence
\begin{eqnarray}
	\nabla_s^2 h(\mathbf{s}, \mathbf{w})&=&\sum_{i=1}^m \nabla_s^2  \phi(\norm{\mathbf{s}-\mathbf{X}_i(\mathbf{w})}) \nonumber\\
	&\succeq& \sum_{i\notin \mathcal{B}}\mathbf{1}(\norm{\mathbf{s}-\mathbf{X}_i(\mathbf{w})}\leq T)\mathbf{I}\nonumber\\
	&=&\sum_{i\notin \mathcal{B}}\mathbf{1}(\norm{\mathbf{s}-\mathbf{G}_i(\mathbf{w})}\leq T)\mathbf{I}\nonumber\\
	&\succeq & (1-\epsilon)m \mathbf{I}.
	\label{eq:hess}
\end{eqnarray}
Therefore
\begin{eqnarray}
	\norm{\nabla_sh(a(\mathbf{w}), \mathbf{w}) - \nabla_s h(g(\mathbf{w}), \mathbf{w})}\geq \min\{\norm{g(\mathbf{w}) - a(\mathbf{w})}, r_s \}(1-\epsilon)m.
	\label{eq:gradlb}
\end{eqnarray}
From \eqref{eq:hgrad1} and \eqref{eq:hgrad2}, 
\begin{eqnarray}
	\norm{\nabla_sh(a(\mathbf{w}), \mathbf{w}) - \nabla_s h(g(\mathbf{w}), \mathbf{w})}\leq \epsilon m T.
	\label{eq:gradub}
\end{eqnarray}
Therefore, from \eqref{eq:gradlb} and \eqref{eq:gradub}, 
\begin{eqnarray}
	\min\{\norm{g(\mathbf{w}) - a(\mathbf{w})}, r_s \}\leq \frac{\epsilon T}{1-\epsilon}.
	\label{eq:diffub}
\end{eqnarray}
From \eqref{eq:Tlb} and \eqref{eq:rs}, it can be shown that $r_s=T-2r_0\geq \epsilon T/(1-\epsilon)$. Therefore, \eqref{eq:diffub} becomes
\begin{eqnarray}
	\norm{g(\mathbf{w}) - a(\mathbf{w})}\leq \frac{\epsilon T}{1-\epsilon}.
\end{eqnarray}
The proof of Lemma \ref{lem:diff1} is complete.
\subsection{Proof of Lemma \ref{lem:diff2}}\label{sec:diff2}
Recall that the condition of Lemma \ref{lem:diff2} is the same as Lemma \ref{lem:diff1}, which requires that $T\geq 2r_0/(1-2\epsilon)$. Therefore, for all $i\in [m]$, from Lemma \ref{lem:dist}, 
\begin{eqnarray}
	\norm{a(\mathbf{w}) - \mathbf{G}_i(\mathbf{w})}\leq 2r_0<T.
\end{eqnarray}
Hence
\begin{eqnarray}
	a(\mathbf{w})&=&\underset{s}{\arg\min}\sum_{i=1}^m \phi(\norm{\mathbf{s}-\mathbf{G}_i(\mathbf{w})})\nonumber\\
	&=&\underset{s}{\arg\min}\sum_{i=1}^m \frac{1}{2}\norm{\mathbf{s}-\mathbf{G}_i(\mathbf{w})}^2\nonumber\\
	&=&\frac{1}{m}\sum_{i=1}^m \mathbf{G}_i(\mathbf{w}).
\end{eqnarray}
From \eqref{eq:b2}, 
\begin{eqnarray}
	\norm{a(\mathbf{w}) - \nabla F(\mathbf{w})}\leq \Delta_0.
\end{eqnarray}
\subsection{Proof of Lemma \ref{lem:diffnew}}\label{sec:diffnew}
Define $h(\mathbf{s}, \mathbf{w})$ as in \eqref{eq:hdf}. Then
\begin{eqnarray}
	\nabla_sh(g(\mathbf{w}), \mathbf{w}) = \mathbf{0}.
\end{eqnarray}
Define a unit vector points to the direction of $\nabla F(\mathbf{w}) - g(\mathbf{w})$:
\begin{eqnarray}
	\mathbf{u}=\frac{\nabla F(\mathbf{w}) - g(\mathbf{w})}{\norm{\nabla F(\mathbf{w}) - g(\mathbf{w})}}.
\end{eqnarray}
Then
\begin{eqnarray}
	\mathbf{u}^T\nabla_sh(g(\mathbf{w}), \mathbf{w}) = 0.
	\label{eq:zerograd}
\end{eqnarray}
Denote $r=\norm{\nabla F(\mathbf{w}) - g(\mathbf{w})}$. If $r>T+r_0$, then for all $i\in [m]$,
\begin{eqnarray}
	\norm{g(\mathbf{w})-\mathbf{G}_i(\mathbf{w})}\geq \norm{g(\mathbf{w})-\nabla F(\mathbf{w})}-\norm{\mathbf{G}_i(\mathbf{w}) - \nabla F(\mathbf{w})}\geq r-r_0>T.
	\label{eq:Tbound}
\end{eqnarray}
Hence
\begin{eqnarray}
	\mathbf{u}^T\nabla_s h(g(\mathbf{w}, \mathbf{w})) &=& \sum_{i=1}^m \mathbf{u}^T \nabla_s \phi(\norm{g(\mathbf{w})-\mathbf{X}_i(\mathbf{w})})\nonumber\\
	&=&\sum_{i\in [m]\setminus\mathcal{B}} \mathbf{u}^T \nabla_s \phi(\norm{g(\mathbf{w})-\mathbf{X}_i(\mathbf{w})})+\sum_{i\in \mathcal{B}} \mathbf{u}^T \nabla_s \phi(\norm{g(\mathbf{w})-\mathbf{X}_i(\mathbf{w})})\nonumber\\
	&\overset{(a)}{\geq}& \sum_{i\in [m]\setminus\mathcal{B}} \mathbf{u}^T \nabla_s \phi(\norm{g(\mathbf{w})-\mathbf{G}_i(\mathbf{w})})-\epsilon mT\nonumber\\
	&\overset{(b)}{=}& T\sum_{i\in [m]\setminus \mathcal{B}}\mathbf{u}^T \frac{g(\mathbf{w}) - \mathbf{G}_i(\mathbf{w})}{\norm{g(\mathbf{w})-\mathbf{G}_i(\mathbf{w})}}-\epsilon mT\nonumber\\
	&\overset{(c)}{\geq}& T\sum_{i\in [m]\setminus \mathcal{B}}\sqrt{1-\frac{r_0^2}{r^2}}-\epsilon mT\nonumber\\
	&=&\left[(1-\epsilon)\sqrt{1-\frac{r_0^2}{r^2}}-\epsilon\right]mT. 
	\label{eq:expand}
\end{eqnarray}
For (a), note that $\norm{\nabla_s\phi(\norm{\mathbf{s}-\mathbf{X}_i(\mathbf{w})})}\leq T$ always hold. (b) uses the gradient of $\phi$ in \eqref{eq:phigrad}. From \eqref{eq:Tbound}, $\norm{g(\mathbf{w})-\mathbf{G}_i(\mathbf{w})}\geq T$. For (c), consider the triangle whose vertices are $\nabla F(\mathbf{w}), g(\mathbf{w}), \mathbf{G}_i(\mathbf{w})$, and edges lengths are $\norm{\nabla F(\mathbf{w})-g(\mathbf{w})}=r$, $\norm{\mathbf{G}_i(\mathbf{w}) - \nabla F(\mathbf{w})}\leq r_0$, the cosine of the angle corresponding to vertex $g(\mathbf{w})$ is at least $\sqrt{1-r_0^2/r^2}$. 

From \eqref{eq:zerograd} and \eqref{eq:expand},
\begin{eqnarray}
	(1-\epsilon)\sqrt{1-\frac{r_0^2}{r^2}}-\epsilon\leq 0,
\end{eqnarray}
which yields
\begin{eqnarray}
	r\leq r_0\frac{1-\epsilon}{\sqrt{1-2\epsilon}}.
	\label{eq:rb1}
\end{eqnarray}
Recall that \eqref{eq:rb1} is derived under the condition $r>T+r_0$. Therefore, the following bound holds in general:
\begin{eqnarray}
	\norm{g(\mathbf{w}) - \nabla F(\mathbf{w})}\leq \max\left\{\frac{1-\epsilon}{\sqrt{1-2\epsilon}}r_0, T+r_0 \right\}.
\end{eqnarray}
\section{Proof of Theorem 2}\label{sec:unbalanced}
Let
\begin{eqnarray}
	R=\sqrt{8\sigma^2 \ln \frac{2\time 6^d m C_W(NL)^d}{\delta}},
	\label{eq:Rdef}
\end{eqnarray}	
The complete statement of Theorem 2 is shown as following.
\begin{theorem}\label{thm:unbalanced_app}
	Let
	\begin{eqnarray}
		T_i = T_0+\frac{M}{\sqrt{n_i}},
	\end{eqnarray}
	with $M\geq 2R$, and
	\begin{eqnarray}
		T_0\geq \frac{\epsilon(M+2R)}{1-2\epsilon}\sqrt{\frac{m}{N}},
		\label{eq:mint0}
	\end{eqnarray}
	Then under Assumption 2 and 3, the following equations hold with probability $1-\delta$.
	
	(1) (Strong convex $F$) Under Assumption 1(a), with $\eta\leq 1/L$,
	\begin{eqnarray}
		\norm{\mathbf{w}_t-\mathbf{w}^*}\leq (1-\rho)^t \norm{\mathbf{w}_0-\mathbf{w}^*}+\frac{2\Delta_A}{\mu},
	\end{eqnarray}
	in which $\rho = \eta\mu/2$;
	
	(2) (General convex $F$) Under Assumption 1(b), with $\eta=1/L$, after $t_m=(L/\Delta_B)\norm{\mathbf{w}_0-\mathbf{w}^*}_2$ steps,
	\begin{eqnarray}
		F(\mathbf{w}_{t_m})-F(\mathbf{w}^*)\leq 16\norm{\mathbf{w}_0-\mathbf{w}^*}\Delta_B;
	\end{eqnarray}
	
	(3) (Non-convex $F$) Under Assumption 1(c), with $\eta = 1/L$, after
	$t_m=(2L/\Delta_B^2)(F(\mathbf{w}_0) - F(\mathbf{w}^*))$ steps, we have
	\begin{eqnarray}
		\underset{t=0,1,\ldots, t_m}{\min}\norm{\nabla F(\mathbf{w}_t)}\leq \sqrt{2}\Delta_B,
	\end{eqnarray}
	in which
	\begin{eqnarray}
		\Delta_B = \frac{\epsilon}{1-\epsilon}T_0+\frac{\epsilon(M+2R)}{1-\epsilon}\sqrt{\frac{m}{N}}+\Delta_0.
	\end{eqnarray}
\end{theorem}
Now we prove Theorem \ref{thm:unbalanced}. Define
\begin{eqnarray}
	\mathbf{G}_i(\mathbf{w}) = \frac{1}{n_i}\sum_{j=1}^{n_i}\nabla f(\mathbf{w}, \mathbf{Z}_{ij})
	\label{eq:gdef}
\end{eqnarray}
If $n_i$ are the same for all $i$, then \eqref{eq:gdef} reduces to
\begin{eqnarray}
	\mathbf{G}_i(\mathbf{w}) = \frac{1}{n}\sum_{j=1}^n \nabla f(\mathbf{w}, \mathbf{Z}_{ij}).
	\label{eq:gsimple}
\end{eqnarray}
Define
\begin{eqnarray}
	r_{0i} = \max\left\{\sqrt{\frac{8\sigma^2}{n_i}\ln \frac{2\times 6^d m C_W(NL)^d}{\delta}}, \frac{4\sigma}{n_i} \ln \frac{2\times 6^d m C_W(NL)^d}{\delta}\right\}+\frac{2}{N}.
	\label{eq:r0i}
\end{eqnarray}
Then the following lemma holds.
\begin{lem}\label{lem:hpb2}
	With probability at least $1-\delta$, the following inequalities hold: (1) for all $i\in [m]$ and $\mathbf{w}\in \mathcal{W}$,
	\begin{eqnarray}
		\norm{\mathbf{G}_i(\mathbf{w}) - \nabla F(\mathbf{w})}\leq r_{0i};
		\label{eq:b1new}
	\end{eqnarray}
	
	(2) For all $\mathbf{w}\in \mathcal{W}$,
	\begin{eqnarray}
		\norm{\frac{1}{N}\sum_{i=1}^m n_i\mathbf{G}_i(\mathbf{w}) - \nabla F(\mathbf{w})}\leq \Delta_0,
		\label{eq:b2new}
	\end{eqnarray}
	in which $\Delta_0$ is defined in \eqref{eq:delta0}.
\end{lem}
The proof of Lemma \ref{lem:hpb2} just follows the proof of Lemma \ref{lem:hpb}. We omit the detailed steps here.

Similar to \eqref{eq:adef} for balanced case, define
\begin{eqnarray}
	a(\mathbf{w}) = \underset{s}{\arg\min} \sum_{i=1}^m n_i\phi_i(\norm{\mathbf{s}-\mathbf{G}_i(\mathbf{w})}).
\end{eqnarray}
Then the following lemmas hold.
\begin{lem}\label{lem:diff3}
	Under the condition that \eqref{eq:b1new} and \eqref{eq:b2new} hold uniformly for all $i$ and $\mathbf{w}\in \mathcal{W}$, if $T_i=T_0+M/\sqrt{n_i}$, with $M\geq 2R$, and
	\begin{eqnarray}
		T_0\geq \frac{\epsilon(M+2R)}{1-2\epsilon}\sqrt{\frac{m}{N}},
		\label{eq:T0lb}
	\end{eqnarray}
	then for all $\mathbf{w}\in \mathcal{W}$,
	\begin{eqnarray}
		\norm{g(\mathbf{w}) - a(\mathbf{w})}\leq \frac{\epsilon}{1-\epsilon}T_0+\frac{\epsilon(M+2R)}{1-\epsilon}\sqrt{\frac{m}{N}}.
	\end{eqnarray}
\end{lem}
\begin{proof}
	The proof is shown in Section \ref{sec:diff3}.
\end{proof}
\begin{lem}\label{lem:diff4}
	Under the same conditions as Lemma \ref{lem:diff3}, for all $\mathbf{w}\in \mathcal{W}$,
	\begin{eqnarray}
		\norm{a(\mathbf{w}) - \nabla F(\mathbf{w})}\leq \Delta_0.
	\end{eqnarray}
\end{lem}
The proof of Lemma \ref{lem:diff4} just follows the proof of Lemma \ref{lem:diff2} in Section \ref{sec:diff2}. We omit the detailed steps.

Combine Lemma \ref{lem:diff3} and \ref{lem:diff4}, for all $\mathbf{w}\in \mathcal{W}$,
\begin{eqnarray}
	\norm{g(\mathbf{w}) - \nabla F(\mathbf{w})}\leq \frac{\epsilon}{1-\epsilon}T_0+\frac{\epsilon(M+2R)}{1-\epsilon}\sqrt{\frac{m}{N}}+\Delta_0:=\Delta_B.
\end{eqnarray}
Theorem 2 can then be proved using Lemma \ref{lem:strong}, \ref{lem:convex} and \ref{lem:nonconvex} with $\Delta=\Delta_B$. 
\subsection{Proof of Lemma \ref{lem:diff3}}\label{sec:diff3}
Recall that
\begin{eqnarray}
	g(\mathbf{w}) = \underset{s}{\arg\min}\sum_{i=1}^m n_i\phi_i(\norm{s-\mathbf{X}_i(\mathbf{w})}).
\end{eqnarray}
Define
\begin{eqnarray}
	h(\mathbf{s}, \mathbf{w})=\sum_{i=1}^m n_i\phi_i(\norm{\mathbf{s}-\mathbf{X}_i(\mathbf{w})}).
\end{eqnarray}
Then
\begin{eqnarray}
	g(\mathbf{w}) = \underset{s}{\arg\min} h(\mathbf{s}, \mathbf{w}),
\end{eqnarray}
and
\begin{eqnarray}
	\nabla_s h(g(\mathbf{w}), \mathbf{w}) = \mathbf{0},
	\label{eq:hstable}
\end{eqnarray}
and similar to the derivation of \eqref{eq:hgrad2}, we get
\begin{eqnarray}
	\norm{\nabla_s h(a(\mathbf{w}), \mathbf{w})} &=& \norm{\sum_{i=1}^m n_i\nabla_s \phi_i(\norm{a(\mathbf{w}) - \mathbf{X}_i(\mathbf{w})})}\nonumber\\
	&=&\norm{\sum_{i=1}^m n_i\left(\nabla_s \phi_i(\norm{a(\mathbf{w}) - \mathbf{X}_i(\mathbf{w})}) - \nabla_s \phi_i(\norm{a(\mathbf{w}) - \mathbf{G}_i(\mathbf{w})})\right)}\nonumber\\
	&=&\norm{\sum_{i\in \mathcal{B}} n_i\left(\nabla_s \phi_i(\norm{a(\mathbf{w}) - \mathbf{X}_i(\mathbf{w})}) - \nabla_s \phi_i(\norm{a(\mathbf{w}) - \mathbf{G}_i(\mathbf{w})})\right)}\nonumber\\
	&\leq & \sum_{i\in \mathcal{B}}n_i(T_i+2r_{0i}).
	\label{eq:ha}
\end{eqnarray}
Now bound the Hessian $\nabla_s^2 h(\mathbf{s}, \mathbf{w})$:
\begin{eqnarray}
	\nabla_s^2 h(\mathbf{s}, \mathbf{w})\succeq \sum_{i=1}^m n_i\mathbf{1}(\norm{\mathbf{s}-\mathbf{X}_i(\mathbf{w})}\leq T_i)\mathbf{I}.
\end{eqnarray}
For $s\in B(a(\mathbf{w}), T_0)$ and $i\notin \mathcal{B}$, $\mathbf{X}_i(\mathbf{w}) = \mathbf{G}_i(\mathbf{w})$, hence
\begin{eqnarray}
	\norm{\mathbf{s}-\mathbf{X}_i(\mathbf{w})} \leq \norm{\mathbf{s}-a(\mathbf{w})}+\norm{a(\mathbf{w})-\mathbf{G}_i(\mathbf{w})}\leq T_0+2r_{0i}.
\end{eqnarray}
Let 
\begin{eqnarray}
	n_0=2\sigma \ln \frac{2\times 6^d m C_W(NL)^d}{\delta}.
\end{eqnarray}
For all $i$ with $n_i\geq n_0$, we have $r_{0i}=R/\sqrt{n_i}$. Then
\begin{eqnarray}
	\norm{\mathbf{s}-\mathbf{X}_i(\mathbf{w})}\leq T_0+\frac{2R}{\sqrt{n_i}}\leq T_0+\frac{M}{\sqrt{n_i}}=T.
\end{eqnarray}
Therefore, for all $s\in B(a(\mathbf{w}), T_0)$,
\begin{eqnarray}
	\nabla_s^2 h(\mathbf{s}, \mathbf{w})&\succeq & \sum_{i\notin \mathcal{B}}n_i\mathbf{1}(\norm{\mathbf{s}-\mathbf{X}_i(\mathbf{w})}\leq T_i)\nonumber\\
	&=&\sum_{i\notin \mathcal{B}} n_i\mathbf{I}\nonumber\\
	&\succeq & N(1-\epsilon) \mathbf{I}.
	\label{eq:hess2}
\end{eqnarray}
From \eqref{eq:hstable}, \eqref{eq:ha} and \eqref{eq:hess2}, we have
\begin{eqnarray}
	\min\{\norm{g(\mathbf{w}) - a(\mathbf{w})}, T_0\}\leq \frac{\sum_{i\in \mathcal{B}}n_i(T_i+2r_{0i})}{N(1-\epsilon)}.
	\label{eq:diff}
\end{eqnarray}
Note that
\begin{eqnarray}
	\sum_{i\in \mathcal{B}}n_i(T_i+2r_{0i})&\leq& \sum_{i\in \mathcal{B}}n_i\left(T_0+\frac{M}{\sqrt{n_i}} + \frac{2R}{\sqrt{n_i}}\right)\nonumber\\
	&\overset{(a)}{\leq} & \epsilon N T_0+(M+2R)\sum_{i\in \mathcal{B}}\sqrt{n_i}\nonumber\\
	&\overset{(b)}{\leq} &\epsilon N T_0+\epsilon(M+2R)\sqrt{mN}.
\end{eqnarray}
For (a), recall that $\epsilon$ is already defined in \eqref{eq:eps},
which ensures that $\sum_{i\in \mathcal{B}}n_i\leq \epsilon N$. For (b), according to \eqref{eq:eps}, $|\mathcal{B}|\leq \epsilon m$, thus from Cauchy inequality,
\begin{eqnarray}
	\sum_{i\in \mathcal{B}}\sqrt{n_i}\leq \sqrt{|\mathcal{B}|\sum_{i\in \mathcal{B}} n_i}\leq \sqrt{\epsilon m\cdot\epsilon N}=\epsilon\sqrt{mN}.
\end{eqnarray}
Therefore,
\begin{eqnarray}
	\frac{\sum_{i\in \mathcal{B}}n_i(T_i+2r_{0i})}{N(1-\epsilon)}&\leq& \frac{\epsilon N T_0+\epsilon(M+2R)\sqrt{mN}}{N(1-\epsilon)}\nonumber\\
	&=&\frac{\epsilon}{1-\epsilon}T_0+\frac{\epsilon(M+2R)}{1-\epsilon}\sqrt{\frac{m}{N}}\nonumber\\
	&\overset{(a)}{\leq} & \frac{\epsilon}{1-\epsilon} T_0+\frac{1-2\epsilon}{1-\epsilon}T_0\nonumber\\
	&=& T_0,
	\label{eq:diffub2}	
\end{eqnarray}
in which (a) comes from \eqref{eq:T0lb}. From \eqref{eq:diff} and \eqref{eq:diffub2}, 
\begin{eqnarray}
	\norm{g(\mathbf{w})-a(\mathbf{w})}\leq \frac{\sum_{i\in \mathcal{B}}n_i(T_i+2r_{0i})}{N(1-\epsilon)}\leq \frac{\epsilon}{1-\epsilon}T_0+\frac{\epsilon(M+2R)}{1-\epsilon}\sqrt{\frac{m}{N}}.
\end{eqnarray}

\section{Proof of Theorem 3}\label{sec:niidpf}
Recall that
\begin{eqnarray}
	r_\mu = 4\sigma_\mu \ln \frac{2\times 6^d m C_W(NL)^d}{\delta}+\frac{2}{N},
	\label{eq:rmu}
\end{eqnarray}
and
\begin{eqnarray}
	\Delta_\mu = \max\left\{\sqrt{\frac{8\sigma_\mu^2 \sum_{i=1}^m n_i^2}{N^2}\ln \frac{2\times 6^d C_W(NL)^d}{\delta}}, \frac{4\sigma_\mu n_{\max}}{N}\ln \frac{2\times 6^d C_W(NL)^d}{\delta} \right\}+\frac{2}{N}.
	\label{eq:deltamu}
\end{eqnarray}
The non-asymptotic statement of Theorem 3 is shown as follows.
\begin{theorem}\label{thm:niid_app}
	Let
	\begin{eqnarray}
		T_i = T_0+\frac{M}{\sqrt{n_i}},
	\end{eqnarray}
	with $M\geq 2R$, and
	\begin{eqnarray}
		T_0\geq\frac{2r_\mu}{1-2\epsilon}+ \frac{\epsilon(M+2R)}{1-2\epsilon}\sqrt{\frac{m}{N}}.
	\end{eqnarray}
	Under Assumption 2 and Assumption 4, the following equations holds with probability at least $1-2\delta$.
	(1) (Strong convex $F$) Under Assumption 1(a), with $\eta\leq 1/L$,
	\begin{eqnarray}
		\norm{\mathbf{w}_t-\mathbf{w}^*}\leq (1-\rho)^t \norm{\mathbf{w}_0-\mathbf{w}^*}+\frac{2\Delta_C}{\mu},
	\end{eqnarray}
	in which $\rho = \eta\mu/2$;
	
	(2) (General convex $F$) Under Assumption 1(b), with $\eta=1/L$, after $t_m=(L/\Delta_C)\norm{\mathbf{w}_0-\mathbf{w}^*}_2$ steps,
	\begin{eqnarray}
		F(\mathbf{w}_{t_m})-F(\mathbf{w}^*)\leq 16\norm{\mathbf{w}_0-\mathbf{w}^*}\Delta_C;
	\end{eqnarray}
	
	(3) (Non-convex $F$) Under Assumption 1(c), with $\eta = 1/L$, after
	$t_m=(2L/\Delta_C^2)(F(\mathbf{w}_0) - F(\mathbf{w}^*))$ steps, we have
	\begin{eqnarray}
		\underset{t=0,1,\ldots, t_m}{\min}\norm{\nabla F(\mathbf{w}_t)}\leq \sqrt{2}\Delta_C,
	\end{eqnarray}
	in which
	\begin{eqnarray}
		\Delta_C &=& \frac{\epsilon}{1-\epsilon}(T_0+2r_\mu)+\frac{\epsilon(M+2R)}{1-\epsilon}\sqrt{\frac{m}{N}} + \Delta_0+\Delta_\mu.
	\end{eqnarray}
\end{theorem}
Similar to Lemma \ref{lem:hpb} and \ref{lem:hpb2}, we show the following lemmas for non i.i.d case.
\begin{lem}\label{lem:hpb3}
	With probability at least $1-2\delta$, the following inequalities hold:
	
	(1) For all $i\in [m]$ and $\mathbf{w}\in \mathcal{W}$,
	\begin{eqnarray}
		\norm{\mathbf{G}_i(\mathbf{w}) - \mu_i(\mathbf{w})}\leq r_{0i};
		\label{eq:hpb3-1}
	\end{eqnarray}
	
	(2) For all $\mathbf{w}\in \mathcal{W}$,
	\begin{eqnarray}
		\norm{\frac{1}{N}\sum_{i=1}^m n_i\mathbf{G}_i(\mathbf{w})-\frac{1}{N}\sum_{i=1}^m n_i\mu_i(\mathbf{w})}\leq \Delta_0;
		\label{eq:hpb3-2}
	\end{eqnarray}
	
	(3) For all $i\in [m]$ and $\mathbf{w}\in \mathcal{W}$,
	\begin{eqnarray}
		\norm{\mu_i(\mathbf{w})-\nabla F(\mathbf{w})}\leq r_\mu,
		\label{eq:hpb3-3}
	\end{eqnarray}
	in which $r_\mu$ is defined in \eqref{eq:rmu};
	
	(4) For all $\mathbf{w}\in \mathcal{W}$,
	\begin{eqnarray}
		\norm{\frac{1}{N}\sum_{i=1}^m n_i\mu_i(\mathbf{w}) - \nabla F(\mathbf{w})}\leq \Delta_\mu,
		\label{eq:hpb3-4}
	\end{eqnarray}
	in which $\Delta_\mu$ is defined in \eqref{eq:deltamu}.
\end{lem}
\begin{proof}
	The proof is shown in Section \ref{sec:hpb3}.
\end{proof}
\begin{lem}\label{lem:diff5}
	Under the condition that \eqref{eq:hpb3-1}, \eqref{eq:hpb3-2}, \eqref{eq:hpb3-3} and \eqref{eq:hpb3-4} hold uniformly for all $i$ and $\mathbf{w}\in \mathcal{W}$, if $T_i=T_0+M/\sqrt{n_i}$, with $M\geq 2R$, in which $R$ is defined in \eqref{eq:Rdef}, and 
	\begin{eqnarray}
		T_0\geq \frac{2r_\mu}{1-2\epsilon} + \frac{\epsilon(M+2R)}{1-2\epsilon}\sqrt{\frac{m}{N}},
		\label{eq:t0}
	\end{eqnarray}
	then for all $\mathbf{w}\in \mathcal{W}$, 
	\begin{eqnarray}
		\norm{g(\mathbf{w}) - a(\mathbf{w})}\leq \frac{\epsilon}{1-\epsilon} (T_0+2r_\mu) + \frac{\epsilon(M+2R)}{1-\epsilon}\sqrt{\frac{m}{N}}.
	\end{eqnarray}
\end{lem}
\begin{proof}
	The proof is shown in Section \ref{sec:diff5}.
\end{proof}
\begin{lem}\label{lem:diff6}
	Under the same conditions as Lemma \ref{lem:diff5}, for all $\mathbf{w}\in \mathcal{W}$,
	\begin{eqnarray}
		\norm{a(\mathbf{w}) - \nabla F(\mathbf{w})}\leq \Delta_0+\Delta_\mu,
	\end{eqnarray}
	in which $\Delta_0$, $\Delta_\mu$ are defined in \eqref{eq:delta0} and \eqref{eq:deltamu}, respectively.
\end{lem}
Define 
\begin{eqnarray}
	\Delta_C = \frac{\epsilon}{1-\epsilon}(T_0+2r_\mu) + \frac{\epsilon(M+2R)}{1-\epsilon}\sqrt{\frac{m}{N}}+\Delta_0+\Delta_\mu.
\end{eqnarray}
Then for all $\mathbf{w}\in \mathcal{W}$,
\begin{eqnarray}
	\norm{g(\mathbf{w}) - \nabla F(\mathbf{w})}\leq \Delta_C.
\end{eqnarray}
Theorem \ref{thm:niid_app} can then be proved using Lemma \ref{lem:strong}, \ref{lem:convex} and \ref{lem:nonconvex} with $\Delta=\Delta_C$.
\subsection{Proof of Lemma \ref{lem:hpb3}}\label{sec:hpb3}
Similar to Lemma \ref{lem:hpb2}, (1) in Lemma \ref{lem:hpb3} holds, i.e. \eqref{eq:hpb3-1} holds for all $i$ and $\mathbf{w}\in \mathcal{W}$. Now prove (2)-(4) separately.

\textbf{Proof of (2) in Lemma \ref{lem:hpb3}}. For any $\mathbf{w}\in \mathcal{W}$,
\begin{eqnarray}
	\mathbb{E}\left[e^{\lambda \mathbf{v}^T (\mathbf{G}_i(\mathbf{w}) - \mu_i(\mathbf{w}))}|\mu_i(\mathbf{w})\right]=\mathbb{E}\left[\exp\left[\frac{1}{n}\lambda \mathbf{v}^T\sum_{j=1}^{n_i}(\nabla f(\mathbf{w}, \mathbf{Z}_{ij}) - \mu_i(\mathbf{w}))\right]|\mu_i(\mathbf{w})\right].
\end{eqnarray}
Hence
\begin{eqnarray}
	&&\mathbb{E}\left[\exp\left[\lambda \mathbf{v}^T \left(\frac{1}{N}\sum_{i=1}^m n_i\mathbf{G}_i(\mathbf{w}) - \frac{1}{N}\sum_{i=1}^m n_i\mu_i(\mathbf{w})\right)\right]|\mu_1(\mathbf{w}), \ldots, \mu_m(\mathbf{w})\right]\nonumber\\
	&=& \mathbb{E}\left[\exp\left[\frac{1}{N}\lambda \mathbf{v}^T \sum_{i=1}^m \sum_{j=1}^{n_i}(\nabla f(\mathbf{w}, \mathbf{Z}_{ij}) - \mu_i(\mathbf{w}))\right]|\mu_1(\mathbf{w}), \ldots, \mu_m(\mathbf{w})\right]\nonumber\\
	&=&\Pi_{i=1}^m \Pi_{j=1}^{n_i}\mathbb{E}\left[\exp\left[\frac{1}{N}\lambda\mathbf{v}^T(\nabla f(\mathbf{w}, \mathbf{Z}_{ij})-\mu_i(\mathbf{w}))\right]|\mu_i(\mathbf{w})\right]\nonumber\\
	&=&\Pi_{i=1}^m \Pi_{j=1}^{n_i}e^{\frac{1}{2N^2}\sigma^2 \lambda^2} \text{ if } |\lambda|\leq \frac{N}{\sigma}\nonumber\\
	&=& e^{\frac{1}{2N}\sigma^2 \lambda^2} \text{ if } |\lambda|\leq \frac{N}{\sigma}.
\end{eqnarray}
Therefore, by Chernoff bound,
\begin{eqnarray}
	\text{P}\left(\mathbf{v}^T \left(\frac{1}{N}\sum_{i=1}^m n_i\mathbf{G}_i(\mathbf{w}) - \frac{1}{N}\sum_{i=1}^m n_i\mu_i(\mathbf{w})\right) > u\right)&\leq & \underset{0\leq \lambda\leq N/\sigma}{\inf} e^{-\lambda u}e^{\frac{1}{2N}\sigma^2 \lambda^2}\nonumber\\
	&\leq & e^{-\frac{N}{2}\min\left\{\frac{u^2}{\sigma^2}, \frac{u}{\sigma}\right\}}.
\end{eqnarray}
Similar to \eqref{eq:largeprob},
\begin{eqnarray}
	\text{P}\left(\norm{\frac{1}{N}\sum_{i=1}^m n_i\mathbf{G}_i(\mathbf{w})-\frac{1}{N}\sum_{i=1}^m n_i\mu_i(\mathbf{w})} > u\right)\leq 6^d e^{-\frac{N}{2}\min\left\{\frac{u^2}{4\sigma^2}, \frac{u}{2\sigma} \right\}}.
\end{eqnarray}
According to \eqref{eq:delta0},
\begin{eqnarray}
	\text{P}\left(\exists k\in [N_c(l)], \norm{\frac{1}{N}\sum_{i=1}^m n_i\mathbf{G}_i(\mathbf{w}^k)-\frac{1}{N}\sum_{i=1}^m n_i\mu_i(\mathbf{w}^k)}>\Delta_0-\frac{2}{N}\right)\leq \frac{\delta}{2}.
\end{eqnarray}
Following steps similar to \eqref{eq:togrid}, the corresponding bound for all $\mathbf{w}\in \mathcal{W}$ is
\begin{eqnarray}
	\text{P}\left(\exists \mathbf{w}\in \mathcal{W}, \norm{\frac{1}{N}\sum_{i=1}^m n_i\mathbf{G}_i(\mathbf{w})-\frac{1}{N}\sum_{i=1}^m n_i\mu_i(\mathbf{w})}>\Delta_0-\frac{2}{N}\right)\leq \frac{\delta}{2}.	
\end{eqnarray}
The proof of (2) in Lemma \ref{lem:hpb3} is complete.

\textbf{Proof of (3) in Lemma \ref{lem:hpb3}}. Since $\mu_i(\mathbf{w})$ is sub-exponential around $\nabla F(\mathbf{w})$ with parameter $\sigma_\mu$, using similar steps as the proof of (2) in Lemma \ref{lem:hpb3}, it can be shown that
\begin{eqnarray}
	\text{P}\left(\norm{\mu_i(\mathbf{w}) - \nabla F(\mathbf{w})}> u\right)\leq 6^d e^{-\frac{1}{2}\min\left\{\frac{u^2}{4\sigma_\mu^2}, \frac{u}{2\sigma_\mu} \right\}}.
\end{eqnarray}
Therefore,
\begin{eqnarray}
	\text{P}(\exists i\in [m], \exists k \in [N_c(l)], \norm{\mu_i(\mathbf{w}) - \nabla F(\mathbf{w})}> u)\leq 6^d m C_W(NL)^de^{-\frac{1}{2}\min\left\{\frac{u^2}{4\sigma_\mu^2}, \frac{u}{2\sigma_\mu} \right\}}.
\end{eqnarray}
From the definition of $r_\mu$ in \eqref{eq:rmu}, 
\begin{eqnarray}
	\text{P}(\exists i\in [m], \exists k \in [N_c(l)], \norm{\mu_i(\mathbf{w}^k) - \nabla F(\mathbf{w}^k)}> r_\mu)\leq \frac{\delta}{2}.	
\end{eqnarray}
The corresponding bound for all $\mathbf{w}\in \mathcal{W}$ is
\begin{eqnarray}
	\text{P}(\exists i\in [m], \exists \mathbf{w}\in \mathcal{W}, \norm{\mu_i(\mathbf{w}) - \nabla F(\mathbf{w})}> r_\mu)\leq \frac{\delta}{2}.	
\end{eqnarray}

\textbf{Proof of (4) in Lemma \ref{lem:hpb3}}. For any $\mathbf{w}\in \mathcal{W}$,
\begin{eqnarray}
	&&\mathbb{E}\left[\exp\left[\lambda \mathbf{v}^T\left(\frac{1}{N}\sum_{i=1}^m n_i(\mu_i(\mathbf{w})-\nabla F(\mathbf{w}))\right)\right]\right]\nonumber\\
	&=&\Pi_{i=1}^m \mathbb{E}\left[\exp\left(\frac{\lambda}{N}\mathbf{v}^T n_i(\mu_i(\mathbf{w}) - \nabla F(\mathbf{w}))\right)\right]\nonumber\\
	&\overset{(a)}{\leq} & \Pi_{i=1}^m e^{\frac{1}{2}\sigma_\mu^2 \lambda^2 \frac{n_i^2}{N^2}} \text{ if } |\lambda|\leq \frac{N}{n_{\max} \sigma_\mu}\nonumber\\
	&=& e^{\frac{\sigma_\mu^2 \lambda^2}{2N^2}\sum_{i=1}^m n_i^2} \text{ if } |\lambda|\leq \frac{N}{n_{\max} \sigma_\mu}.
\end{eqnarray}
For (a), note that
\begin{eqnarray}
	\mathbb{E}\left[\exp\left(\frac{\lambda}{N}\mathbf{v}^T n_i(\mu_i(\mathbf{w}) - \nabla F(\mathbf{w}))\right)\right]\leq e^{\frac{1}{2}\sigma_\mu^2 \lambda^2 \frac{n_i^2}{N^2}}, \text{ if }|\lambda|\leq \frac{N}{n_i\sigma_\mu}.
\end{eqnarray}
With $|\lambda|\leq N/n_{\max}$, the above bounds hold for all $i$.

Therefore
\begin{eqnarray}
	\text{P}\left(\mathbf{v}^T \left(\frac{1}{N}\sum_{i=1}^m n_i(\mu_i(\mathbf{w}) - \nabla F(\mathbf{w}))\right)> u\right)\leq \underset{0\leq \lambda\leq N/(n_{\max}\sigma_\mu)}{\inf}e^{-\lambda u}e^{\frac{\sigma_\mu^2 \lambda^2}{2N^2}\sum_{i=1}^m n_i^2}
	\label{eq:pb}
\end{eqnarray}
If $u\leq (\sigma_\mu \sum_{i=1}^m n_i^2)/(Nn_{\max})$, then \eqref{eq:pb} reaches maximum with $\lambda = N^2 u/(\sigma_\mu^2 \sum_{i=1}^m n_i^2)\leq N/(n_{\max}\sigma_\mu)$. Then
\begin{eqnarray}
	\text{P}\left(\mathbf{v}^T\sum_{i=1}^m n_i(\mu_i(\mathbf{w}) - \nabla F(\mathbf{w})) > u\right)\leq \exp\left(-\frac{N^2 u^2}{2\sigma_\mu^2 \sum_{i=1}^m n_i^2}\right).
	\label{eq:smallu}
\end{eqnarray}
If $u>(\sigma_\mu \sum_{i=1}^m n_i^2)/(Nn_{\max})$, then \eqref{eq:pb} reaches maximum with $\lambda = N/(n_{\max}\sigma_\mu)$. Then
\begin{eqnarray}
	\text{P}\left(\mathbf{v}^T\sum_{i=1}^m n_i(\mu_i(\mathbf{w}) - \nabla F(\mathbf{w})) > u\right)&\overset{(a)}{\leq}& \exp\left[-\frac{Nu}{n_{\max}\sigma_\mu}+\frac{1}{2n_{\max}^2}\sum_{i=1}^m n_i^2\right]\nonumber\\
	&\overset{(b)}{\leq} & \exp\left[-\frac{Nu}{n_{\max}\sigma_\mu}+\frac{Nu}{2n_{\max}\sigma_\mu}\right]\nonumber\\
	&=&\exp\left[-\frac{Nu}{2n_{\max}\sigma_\mu}\right],
	\label{eq:largeu}
\end{eqnarray}
in which (a) uses \eqref{eq:pb}, (b) uses the condition $u>(\sigma_\mu \sum_{i=1}^m n_i^2)/(Nn_{\max})$, which yields
\begin{eqnarray}
	\frac{\sum_{i=1}^m n_i^2}{n_{\max}}<\frac{Nu}{\sigma_\mu}.
\end{eqnarray}
Hence (b) holds. From \eqref{eq:smallu} and \eqref{eq:largeu},
\begin{eqnarray}
	\text{P}\left(\mathbf{v}^T\sum_{i=1}^m n_i(\mu_i(\mathbf{w}) - \nabla F(\mathbf{w})) > u\right)\leq \exp\left[-\min\left\{\frac{N^2 u^2}{2\sigma_\mu^2 \sum_{i=1}^m n_i^2}, \frac{Nu}{2n_{\max}\sigma_\mu} \right\}\right].
\end{eqnarray}
Similar to \eqref{eq:largeprob},
\begin{eqnarray}
	\text{P}\left(\norm{\sum_{i=1}^m n_i(\mu_i(\mathbf{w}) - \nabla F(\mathbf{w}))} > u\right)\leq 6^d\exp\left[-\min\left\{\frac{N^2 u^2}{8\sigma_\mu^2 \sum_{i=1}^m n_i^2}, \frac{Nu}{4n_{\max}\sigma_\mu} \right\}\right].
\end{eqnarray}
Recall the definition of $\Delta_\mu$ in \eqref{eq:deltamu}, for all $\mathbf{w}\in \mathcal{W}$,
\begin{eqnarray}
	\text{P}\left(\norm{\sum_{i=1}^m n_i(\mu_i(\mathbf{w}) - \nabla F(\mathbf{w}))} > \Delta_\mu\right)\leq \frac{\delta}{2}.	
\end{eqnarray}
Now we have proved that all four conditions in Lemma \ref{lem:hpb3} are violated with probability no more than $\delta/2$. Hence, with probability $1-2\delta$, all these four conditions hold. The proof is complete.

\subsection{Proof of Lemma \ref{lem:diff5}}\label{sec:diff5}
Now \eqref{eq:hstable} still holds, i.e.
\begin{eqnarray}
	\nabla_s h(g(\mathbf{w}), \mathbf{w}) = 0.
\end{eqnarray}
\eqref{eq:ha} becomes
\begin{eqnarray}
	\norm{\nabla_s h(a(\mathbf{w}), \mathbf{w})}\leq \sum_{i\in \mathcal{B}}n_i(T_i+2r_{0i} + 2r_\mu).
\end{eqnarray}
Note that \eqref{eq:hess} still holds. Now define $r_s=T_0-2r_\mu$. For all $\mathbf{s}\in B(a(\mathbf{w}), r_s)$, $i\notin \mathcal{B}$, we have
\begin{eqnarray}
	\norm{\mathbf{s}-\mathbf{X}_i(\mathbf{w})} &=& \norm{\mathbf{s}-\mathbf{G}_i(\mathbf{w})}\nonumber\\
	&\leq & \norm{\mathbf{s}-a(\mathbf{w})} + \norm{a(\mathbf{w}) - \mathbf{G}_i(\mathbf{w})}\nonumber\\
	&\leq & r_s+2r_{0i}+2r_\mu\nonumber\\
	&\leq & T_0+\frac{2R}{\sqrt{n_i}}\nonumber\\
	&\leq & T_0+\frac{2R}{\sqrt{n_i}}\nonumber\\
	&\leq &T.
\end{eqnarray}
The last step comes from the statements in Lemma \ref{lem:diff5} that $T_i=T_0+M/\sqrt{n_i}$ and $M\geq 2R$. Therefore, for $s\in B(a(\mathbf{w}), r_s)$,
\begin{eqnarray}
	\nabla_s^2 h(\mathbf{s}, \mathbf{w})\succeq N(1-\epsilon)\mathbf{I}.
\end{eqnarray}
Note that
\begin{eqnarray}
	\norm{\nabla_s h(a(\mathbf{w}), \mathbf{w})} &=& \norm{\nabla_s h(a(\mathbf{w}), \mathbf{w}) - \nabla_s h(g(\mathbf{w}), \mathbf{w})}\nonumber\\
	&\geq & N(1-\epsilon)\min\{\norm{g(\mathbf{w}) - a(\mathbf{w})}, r_s \}.
\end{eqnarray}
Therefore
\begin{eqnarray}
	\min\left\{\norm{g(\mathbf{w}) - a(\mathbf{w})}, r_s \right\}&\leq& \frac{\sum_{i\in \mathcal{B}}n_i(T_i+2r_{0i}+2r_\mu)}{N(1-\epsilon)}\nonumber\\
	&\leq & \frac{\sum_{i\in \mathcal{B}}n_i \left(T_0+\frac{M+2R}{\sqrt{n_i}}+2r_\mu\right)}{N(1-\epsilon)}\nonumber\\
	&\leq & \frac{(T_0+2r_\mu) \epsilon N + \epsilon(M+2R)\sqrt{mN}}{N(1-\epsilon)}\nonumber\\
	&=&\frac{\epsilon}{1-\epsilon}(T_0+2r_\mu) + \frac{\epsilon}{1-\epsilon}(M+2R)\sqrt{\frac{m}{N}}.
\end{eqnarray}
Note that from \eqref{eq:t0}, 
\begin{eqnarray}
	\frac{\epsilon}{1-\epsilon}(T_0+2r_\mu) + \frac{\epsilon}{1-\epsilon}(M+2R)\sqrt{\frac{m}{N}}\leq r_s.
\end{eqnarray}
Therefore
\begin{eqnarray}
	\norm{g(\mathbf{w}) - a(\mathbf{w})}\leq \frac{\epsilon}{1-\epsilon} (T_0+2r_\mu) + \frac{\epsilon(M+2R)}{1-\epsilon}\sqrt{\frac{m}{N}}.
\end{eqnarray}

\subsection{Proof of Lemma \ref{lem:diff6}}\label{sec:diff6}
\begin{eqnarray}
	\norm{a(\mathbf{w}) - \nabla F(\mathbf{w})}&\leq & \norm{\frac{1}{N}\sum_{i=1}^m n_i\mathbf{G}_i(\mathbf{w}) - \nabla F(\mathbf{w})}\nonumber\\
	&\leq & \norm{\frac{1}{N}\sum_{i=1}^m n_i\mathbf{G}_i(\mathbf{w}) - \frac{1}{N}\sum_{i=1}^m n_i\mu_i(\mathbf{w})}+\norm{\frac{1}{N}\sum_{i=1}^m n_i\mu_i(\mathbf{w}) - \nabla F(\mathbf{w})}\nonumber\\
	&\leq &\Delta_0+\Delta_\mu.
\end{eqnarray}

\end{document}